\def\year{2020}\relax
\documentclass[letterpaper]{article} 
\usepackage{aaai20}  
\usepackage{times}  
\usepackage{helvet} 
\usepackage{courier}  
\usepackage[hyphens]{url}  
\usepackage{graphicx} 
\usepackage{graphicx}  
\usepackage{bbm}
\usepackage[dvipsnames]{xcolor}
\usepackage{tikz}
\usetikzlibrary{arrows,calc,intersections}
\usepackage{amsfonts}
\usepackage{amsmath}
\usepackage{amsthm}
\usepackage{amssymb}
\usepackage{stmaryrd}
\usepackage[vlined,ruled,linesnumbered]{algorithm2e}
\usepackage{xfrac}
\usepackage{subcaption}

\usepackage{xspace}

\DeclareRobustCommand{\ie}{i.e.,\@\xspace}

\DeclareRobustCommand{\wrt}{w.r.t.\@\xspace}

\newcommand{\one}[1]{\mathbbm{1} \left(#1 \right)}
\DeclareMathOperator*{\argmax}{\arg\,\max}

\newcommand{\ucb}{{\small\textsc{UCB}}\xspace}
\newcommand{\cucb}{{\small\textsc{CUCB}}\xspace}

\newcommand{\linucb}{{\small\textsc{LinUCB}}\xspace}
\newcommand{\clinucb}{{\small\textsc{CLUCB}}\xspace}
\newcommand{\ccucb}{{\small\textsc{CUCB}2}\xspace}
\newcommand{\cclinucb}{{\small\textsc{CLUCB}2}\xspace}
\newcommand{\cclinucbc}{{\small\textsc{CLUCB2T}}\xspace}
\newcommand{\ccucbc}{{\small\textsc{CUCB2T}}\xspace}

\newcommand{\wt}[1]{\widetilde{#1}}
\newcommand{\wh}[1]{\widehat{#1}}

\tikzstyle{every picture}+=[remember picture]

\usepackage[colorinlistoftodos, textwidth=20mm]{todonotes}

\newtheorem{theorem}{Theorem}

\newtheorem{lemma}{Lemma}

\newtheorem{assumption}{Assumption}

\title{Improved Algorithms for Conservative Exploration in Bandits}
\author{
Evrard Garcelon,\textsuperscript{\rm 1}
Mohammad Ghavamzadeh,\textsuperscript{\rm 1}
Alessandro Lazaric,\textsuperscript{\rm 1}
Matteo Pirotta,\textsuperscript{\rm 1}\\
\textsuperscript{\rm 1}Facebook AI Research,\\
evrard.garcelon@gmail.com, mgh@fb.com, lazaric@fb.com,
pirotta@fb.com
}

\begin{document}

\maketitle

\begin{abstract}
In many fields such as digital marketing, healthcare, finance, and robotics, it is common to have a well-tested and reliable baseline policy running in production (e.g., a recommender system). Nonetheless, the baseline policy is often suboptimal. In this case, it is desirable to deploy online learning algorithms (e.g., a multi-armed bandit algorithm) that interact with the system to learn a better/optimal policy \textit{under the constraint} that during the learning process the performance is almost never worse than the performance of the baseline itself. In this paper, we study the \textit{conservative learning} problem in the contextual linear bandit setting and introduce a novel algorithm, the Conservative Constrained \linucb (\cclinucb). We derive regret bounds for \cclinucb that match existing results and empirically show that it outperforms state-of-the-art conservative bandit algorithms in a number of synthetic and real-world problems. Finally, we consider a more realistic constraint where the performance is verified only at predefined checkpoints (instead of at every step) and show how this relaxed constraint favorably impacts the regret and empirical performance of \cclinucb.
	
\end{abstract}
%
\section{Introduction}
\label{sec:intro}

Many problems in fields such as digital marketing, healthcare, finance, and robotics can be formulated as decision-making under uncertainty. Although many learning algorithms have been developed to find a good/optimal policy for these problems, a major obstacle in using them in real-world applications is the lack of guarantees for the actual performance of the policies they execute over time. Therefore, for the applicability of these algorithms, it is important that they execute policies that are guaranteed to perform at least as well as an existing {\em baseline}. We can think of the baseline either as a baseline value or the performance of a baseline policy. It is important to note that since the learning algorithms generate these polices from data, they are random variables, and thus, all the guarantees on their performance should be in high probability. This problem has been recently studied under the general title of {\em safety w.r.t.~a baseline} in bandits and reinforcement learning (RL), in both {\em offline}~\cite{Bottou13CR,Thomas15HCO,Thomas15HCP,Swaminathan15CR,Petrik16SP} and {\em online}~\cite{mansour2015bayesian,WuSLS16,KazerouniGAR17,Katariya2019interleaving} settings. 

In the online setting, which is the focus of this paper, the learning algorithm updates its policy while interacting with the system. Although the algorithm eventually learns a good or an optimal policy, there is no guarantee on the performance of the intermediate policies, especially at the very beginning, when the algorithm needs to heavily explore different options. Therefore, in order to make sure that at any point in time the (cumulative) performance of the policies generated by the algorithm is not worse than the baseline, it is important to control the exploration and make it more {\em conservative}. Consider a recommender system that runs our learning algorithm. Although we are confident that our algorithm will eventually learn a strategy that performs as well as the baseline, and possibly even better, we should control its exploration not to lose too many customers, as a result of providing them with unsatisfactory recommendations. This setting has been studied in multi-armed bandits~\cite{WuSLS16}, contextual linear bandits~\cite{KazerouniGAR17}, and stochastic combinatorial semi-bandits~\cite{Katariya2019interleaving}. These papers formulate the problem using a constraint defined based on the performance of the baseline policy (mean of the baseline arm in the multi-armed bandit case), and modify the corresponding UCB-type algorithm~\cite{Auer02FA} to satisfy this constraint. At each round, the conservative bandit algorithm computes the action suggested by the corresponding UCB algorithm, if the action satisfies the constraint, it is taken, otherwise, the algorithm acts according to the baseline policy. Another algorithm in the online setting is by~\cite{mansour2015bayesian} that balances exploration and exploitation such that the actions taken are compatible with the agent's (customer's) incentive formulated as a Bayesian prior.

In this paper, we focus on UCB-type algorithms and improve the design and empirical performance of the conservative algorithms in the contextual linear bandit setting. We first highlight the limitations of the existing conservative bandit algorithms~\cite{WuSLS16,KazerouniGAR17} and show that simple modifications in constructing the conservative condition and the arm-selection strategy may significantly improve their performance. We show that our algorithm is formally correct by proving a regret bound, matching existing results and illustrate its practical advantage w.r.t.\ state-of-the-art algorithms in a number of synthetic and real-world environments. Finally, we consider the more realistic scenario where the conservative constraint is verified at predefined checkpoints (e.g., a manager may be interested in verifying the performance of the learning algorithm every few days). In this case, we prove a regret bound showing that as the checkpoints become less frequent, the conservative condition has less impact on the regret, which eventually reduces to the standard (unconstrained) one.
\section{Conservative Contextual Linear Bandits}
\label{sec:clucb}

We consider the standard linear bandit setting. At each time $t$, the agent selects an arm $a_t \in \mathcal{A}_t$ and observes a reward 
\begin{align}\label{eq:linear.bandit}
        r_{a}^t = \left\langle \theta^{\star}, \phi_{a}^t \right\rangle + \eta_{a}^{t}
:= \mu_a^t + \eta_{a}^{t},
\end{align}
where $\theta^{\star} \in \mathbb{R}^{d}$ is a parameter vector, $\phi_{a}^t \in \mathbb{R}^{d} $ are the features of arm $a$ at time $t$, and $\eta_{a}^{t}$ is a zero-mean  $\sigma^2$-subgaussian noise. 
When the features correspond to the canonical basis, this formulation reduces to multi-armed bandit (MAB) with $d$ arms. In the more general case, the features may depend on a context $x_t$, so that $\phi_a^t = \phi(x_t, a)$ denotes the feature vector of a context-action pair $(x_t, a)$ and~\eqref{eq:linear.bandit} defines the so-called linear contextual bandit setting.

We rely on the following standard assumption on the features and the unknown parameter $\theta^\star$.
\begin{assumption}\label{asm:environment} 
        There exist $B,D \geq 0$, such that $\|\theta^\star\|_2 \leq B$, $\|\phi_a^t\| \leq D$, and $\langle \theta^\star, \phi_a^t \rangle \in [0,1]$, for all $t$ and $a$.
\end{assumption}
Given a finite horizon $n$, the performance of the agent is measured by its (pseudo)-\emph{regret}:
\[
        R(n) = \sum_{t=1}^n \langle \theta^\star, \phi_{a^\star}^t \rangle - \langle \theta^\star, \phi_{a_t}^t \rangle ,
\]
where $a^{\star}_{t} \in \argmax_{a} \langle \theta^\star, \phi_{a}^t \rangle$ is the optimal action at time $t$.
In the conservative setting, the objective is to minimize the regret under additional performance constraints w.r.t.\ a known baseline. We assume the agent has access to a \emph{baseline policy}, which selects action $b_t$ at time $t$.\footnote{In the non-contextual case, the baseline policy reduces to a single baseline action $b$.} 
The learning problem is constrained such that, at any time $t$, the difference in performance (\ie expected cumulative reward) between the baseline and the agent should never fall below a predefined fraction of the baseline performance.
Formally, the \emph{conservative constraint} is given by
\begin{equation}\label{eq:clucb_constr}
        \forall t > 0, \qquad
        \sum_{i = 1}^{t}\mu_{a_{i}}^i \geq (1 - \alpha) \sum_{i=1}^t \mu_{b_i}^i,
\end{equation}
where $\alpha \in (0,1)$ is the conservative level.
As the LHS of~\eqref{eq:clucb_constr} is a random variable depending on the agent's strategy, we require this constraint to be satisfied with high probability.
Finally, in order to keep the presentation and analysis simple, we rely on the following assumption.

\begin{assumption}\label{asm:basline}
For any $t>0$, the performance of the baseline strategy until $t$ is known, i.e., $\sum_{i=1}^t \mu_{b_i}^i$ can be evaluated by the agent.
\end{assumption}

This assumption is often reasonable since the baseline performance can be estimated from historical data (see Rem. 3 in~\cite{KazerouniGAR17}). Furthermore, as shown in~\cite{WuSLS16,KazerouniGAR17}, this knowledge can be removed and the algorithm can be modified to incorporate the estimation process (and preserve the same order of regret). 

\begin{figure}[t]
	\includegraphics[width=\columnwidth]{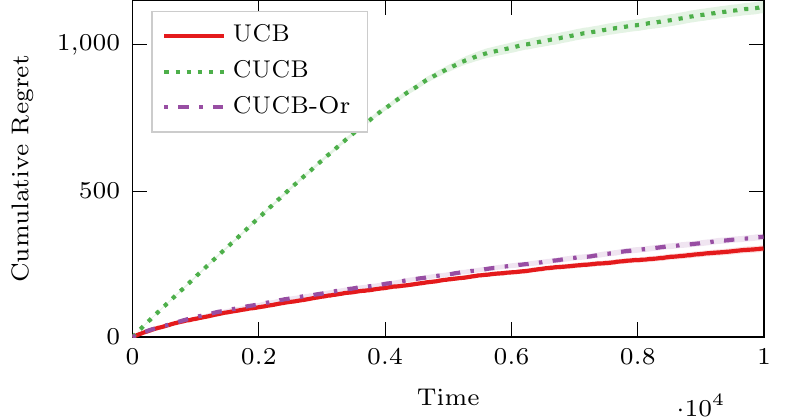}
	\caption{Comparison of the cumulative regret between \ucb, its conservative variant (\cucb), and an \textit{oracle} version of \cucb, where~\eqref{eq:clucb_constr} can be evaluated exactly to decide whether to select the UCB arm or the baseline.}
	\label{fig:oracle_mab}
\end{figure}


\noindent
\textbf{Conservative Exploration.} Conservative exploration algorithms~\cite{WuSLS16,KazerouniGAR17} are based on a two-step process to select the action to play. 
In the first step, they compute an optimistic action based on the optimism-in-the-face-of-uncertainty principle, \ie using \ucb~\cite{auer2002finite} or \linucb~\cite{li2010contextual,abbasi2011improved}, which is effective in exploring and minimizing the regret over time. 
In the second step, they evaluate the conservative condition by replacing the unknown mean with a statistical lower bound.
If the condition is verified, they play the optimistic action, otherwise, they act conservatively by selecting the baseline $b_t$. Playing the baseline over multiple steps contributes to ``build a conservative budget'', so that condition \eqref{eq:clucb_constr} is more likely to be satisfied by the UCB arm, and thus, allowing to execute explorative actions.

Formally, let $S^b_{t}$ be the set of times up to $t$ (included) where the agent played the baseline and $S_{t-1} = [t] \setminus S^b_{t}$ be the complementary set, i.e., when the agent played the \ucb action. 
\clinucb uses the information collected when playing non-conservatively to build an estimator of $\theta^{\star}$ by solving a regularized least-square problem
$ \wh{\theta}_t = (\Phi_t \Phi_t^\top + \lambda I)^{-1} \Phi_t Y_t$, 
where $\lambda > 0$, $\Phi_t = (\phi_{a_i}^i)_{i \in S_{t-1}} \in \mathbb{R}^{d \times |S_{t-1}|}$ and $Y_t = (r_{a_i}^i)_{i \in S_{t-1}} \in \mathbb{R}^{|S_{t-1}|}$.
Denote by $V_t = \lambda I + \Phi_t \Phi_t^\top$ the design matrix of the regularized least-square problem and by $\|x\|_{V} = \sqrt{x^\top V x}$ the weighted norm \wrt any positive matrix $V \in \mathbb{R}^{d \times d}$.
We define the confidence set
$\Theta_t = \{\theta \in \mathbb{R}^d \,:\, \|\theta - \widehat{\theta}_t\|_{V^{-1}_t} \leq \beta_t\}$
where 
\begin{equation}
        \label{eq:beta_linucb}
        \beta_t = \sigma\sqrt{d\log\left( \frac{1 + D^2(1+|S_{t-1}|)/\lambda }{\delta} \right)} + B\sqrt{\lambda},
\end{equation}
which guarantees that $\theta^{\star}\in \Theta_t$, for all $t>0$, w.p.\ $1-\delta$.

Similar to \linucb, the optimistic action is computed as
\[
        a_t \in \argmax_{a \in \mathcal{A}_t} \max_{\theta \in \Theta_t} \langle \theta, \phi_{a}^t \rangle,
\]
and \clinucb decides if the action is ``safe'' by evaluating the following conservative condition:

\begin{small}
\begin{equation}
        \label{eq:clucb_condition}
\sum_{i \in S^b_{t-1}} \mu_{b_i}^i +
        \min_{\theta \in \Theta_t} 
        \Big \langle \theta, \phi_{a_t}^t + \sum_{i \in S_{t-1}} \phi_{a_i}^i \Big \rangle 
        \geq (1-\alpha) \sum_{i = 1}^t \mu_{b_i}^i.
\end{equation}
\end{small}

        The leftmost term in~\eqref{eq:clucb_condition} represents the expected cumulative reward associated to baseline actions played up to time $t-1$. The second term --minimization problem-- denotes the lower bound to the cumulative reward of optimistic actions, including the current optimistic action $a_t$.
        This lower bound is constructed using the most recent confidence set $\Theta_t$.
        The rightmost term denotes the expected cumulative reward of playing the baseline policy at each step.
        This inequality is a surrogate for the conservative condition in~\eqref{eq:clucb_constr}.

\noindent If the condition is satisfied, then the optimistic action $a_t$ is played, otherwise, the baseline strategy is selected and the corresponding action $b_t$ is executed.

\paragraph{Limitations.}
While \clinucb enjoys strong regret guarantees (in the MAB setting, it is indeed near-optimal), its empirical behavior is often over-conservative, i.e.,~the baseline strategy is selected for a very long time to build enough {\em conservative budget} before the actual exploration takes place. We identify two main algorithmic causes for such behavior. 

First, in building the conservative condition~\eqref{eq:clucb_condition}, \cucb and \clinucb rely on possibly loose statistical lower-bounds for the mean of the actions selected so far. This is well illustrated by the simulation in Fig.~\ref{fig:oracle_mab} in the MAB setting, where we report the performance of \ucb, \cucb, and an \textit{oracle} variant of \cucb, when the conservative condition~\eqref{eq:clucb_constr} is evaluated exactly (i.e., no lower-bound is used). 
While the \textit{oracle} version has almost the same regret as \ucb, and thus, showing that the conservative condition itself does not have a major impact on the exploration of \ucb, \cucb has a much higher regret. This shows that possibly loose estimates of the conservative condition have a significant impact on the regret. 
In fact, tightening the conservative condition would allow selecting the baseline strategy only when it is ``strictly'' needed, thus, reducing the conservative steps and improve the overall exploration performance.

Second, the two-step selection strategy of \cucb and \clinucb performs either an exploration step, when the UCB action is selected, or a conservative step, when the baseline is executed. Such sharp division between exploration and conservative steps may be unnecessary, as other actions may still be ``safe'' (i.e., satisfying the conservative condition), and thus, contribute to build ``conservative budget'', and, at the same time, be useful for exploration (e.g., optimistic), despite not being the UCB action. Exploiting such arms may lead to a better performance.

Finally, the conservative condition~\eqref{eq:clucb_constr} itself is often too strict in practice. Instead of performing almost as well as the baseline \textit{at every step}, it is more likely that recurrent ``checkpoints'' are set at which the agent is required to meet the condition. In this case, the agent may have extra time to perform exploratory actions and possibly recover from bad past choices when getting close to the conservative checkpoint.

In the next section, we address the two algorithmic limitations described above, while we later illustrate how a relaxed conservative condition may indeed allow an agent to achieve much smaller regret.



\section{Improved Conservative Exploration}
In this section we present Conservative Constrained \linucb (\cclinucb) (Alg.~\ref{alg:clucb2} with $T=1$), an improved conservative exploration algorithm for contextual linear bandit. 
All the proofs can be found in the extended version.


\begin{figure}[t]
                \begin{minipage}{.44\columnwidth}
        \begin{tikzpicture}
                \coordinate (a1) at (0,0);
                \coordinate (a2) at (20pt,10pt);
                \coordinate (a3) at (60pt,-25pt);
                \coordinate (a4) at (40pt,5pt);
                \coordinate (a5) at (80pt,-40pt);
                \draw[|-|, blue] ($(a1)-(0,45pt)$) -- +(0,90pt);
                \draw[|-|, ForestGreen] ($(a2)-(0,25pt)$) -- +(0,50pt);
                \draw[|-|] ($(a4)-(0,12pt)$) -- +(0,24pt);
                \draw[|-|] ($(a5)-(0,5pt)$) -- +(0,10pt);
                \draw[fill=blue, blue] (a1) circle(2pt) node[right] {$\mu_1$};
                \draw[fill=ForestGreen, ForestGreen] (a2) circle(2pt) node[right] {$\mu_2$};
                \draw[fill=black] (a4) circle(2pt) node[right] {$\mu_3$};
                \draw[fill=red, red] (a3) circle(2pt) node[below, shift={(-10pt,-4pt)}, red] {$\mu_4 := \mu_b$};
                \draw[fill=black] (a5) circle(2pt) node[right] {$\mu_5$};
        \end{tikzpicture}
        \end{minipage}
        \hfill
        \begin{minipage}{.52\columnwidth}
                \begin{tikzpicture}[scale=0.4]
                        \coordinate (elcenter) at (6,6);
                        \coordinate (a1) at (9,9);
                        \coordinate (a2) at (10,7);
                        \coordinate (a3) at (4.5,7.5);
                        \coordinate (a4) at (2.8,2.8);
                        \coordinate (a5) at (0.8,2.1);
                        \draw[help lines, color=gray!30, dashed] (-0.1,-0.1) grid (9.9,9.9);
                        \draw[->] (-0.1,0)--(10.5,0);
                        \draw[->] (0,-0.1)--(0,10);
                        \draw[name path=ellipse,rotate=-45, Lavender, fill=Lavender!6] (elcenter) ellipse (1 and 3);
                        \path[name path=l1] ($(elcenter) + (-1,1)$) -- ($(elcenter)$);
                        \path[name path=l2] ($(elcenter) + (1,-1)$) -- ($(elcenter)$);
                        \path[name intersections={of = ellipse and l2, by=E1}];
                        \path[name intersections={of = ellipse and l1, by=E2}];
                        \draw[-, Lavender] (E1) -- (E2);
                        \path[name path=l1] (0,0) -- (elcenter);
                        \path[name path=l2] (9,9) -- (elcenter);
                        \path[name intersections={of = ellipse and l2, by=E1}];
                        \path[name intersections={of = ellipse and l1, by=E2}];
                        \draw[-, Lavender] (E1) -- (E2);

                        \draw[-, ForestGreen] ($(0,0)!(E2)!(a2)$) -- (E2);

                        \draw[dashed,ForestGreen] (0,0) -- (a2);
                        \draw[-,red] (0,0) -- (a4);
                        \path[name path=l1] (0,0) -- (12,9.3); 
                        \path[name path=l2] (7,4) -- (7,8); 
                        \path[name intersections={of = l1 and l2, by=E1}];
                        \draw[-] (0,0) -- (E1);
                        \path[name path=l1] (0,0) -- (8,10.3); 
                        \path[name path=l2] (5.5,4) -- (5.5,8); 
                        \path[name intersections={of = l1 and l2, by=E1}];
                        \draw[-] (0,0) -- (E1);

                        \draw[Orchid,fill=Orchid] (elcenter) circle(3pt) node[above,xshift=3pt] {$\widehat{\theta}$};
                        \draw[blue, fill=blue] (a1) circle(3pt) node[above] {$a_1$};
                        \draw[ForestGreen, fill=ForestGreen] (a2) circle(3pt) node[above] {$a_2$};
                        \draw[red, fill=red] (a4) circle(3pt) node[above] {$a_4 = b$};
                        \draw[black, fill=black] (a3) circle(3pt) node[above] {$a_3$};
                        \draw[black, fill=black] (a5) circle(3pt) node[above] {$a_5$};
                \end{tikzpicture}
        \end{minipage}
        \caption{Examples of settings where the UCB arm (blue) does not satisfy the conservative condition but there is another ``safe'' arm to play (rather than the baseline).}
        \label{fig:issue_opt}
\end{figure}
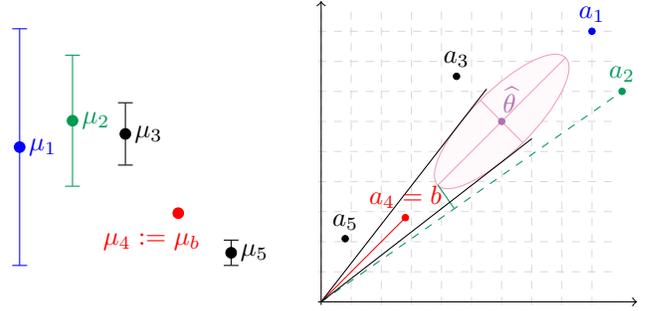

\subsection{\cclinucb}
The first improvement w.r.t.\ \clinucb is relative to the conservative condition~\eqref{eq:clucb_condition}. When evaluating the rewards accumulated by the agents so far, we rely on the fact that the sequence $(r_{a_i}^i - \mu_{a_i}^i)_{i \in S_{t-1}}$ is a Martingale Difference Sequence (MDS) with respect to the filtration $\mathcal{F}_{t-1} = \sigma\left( (a_{j}, \phi_{a_j}^j, r_{a_j}^j)_{j \in S_{t-1}}\right)$, \ie the history strictly before time $t$.
Indeed, the choice of arm $a_{i}$ is $\mathcal{F}_{i}$-measurable and for all $i \in S_{t-1}$ :
\begin{align*}
\mathbb{E}\left[ r_{a_{i}}^i - \mu_{a_{i}}^i \mid \mathcal{F}_{i} \right] = \sum_{ j = 1}^{|\mathcal{A}_i|} \one{ a_{i} = j } \left( \mathbb{E}\left[ r_{j}^i\right] - \mu_{j}^i \right) = 0
\end{align*} 
By using Freedman's inequality~\cite{freedman1975tail} for martingales, with probability at least $1-\delta$ we have
\begin{align}
\Big|\!\!\sum_{i \in S_{t-1}} (r_{a_i}^i - \mu_{a_i}^i) \Big| \leq \psi_L(t) \!:=
 \sigma \sqrt{2|S_{t-1}| L_t^\delta} + \frac{2}{3} L_t^\delta 
\label{eq:martingale}
\end{align}
where $L_t^\delta := \log \left(3\sfrac{\big(\big|S_{t-1}\big| \lor 1\big)^2}{\delta} \right)$. 
Thus we replace~\eqref{eq:clucb_condition} by
%

\begin{small}
\begin{equation}
\label{eq:clucb_condition2}
\sum_{i \in S^b_{t-1}} \mu_{b_i}^i + \sum_{i \in S_{t-1}} r_{a_i}^i - \psi_L(t) + 
\min_{\theta \in \Theta_t} 
\big \langle \theta, \phi_{a_t}^t \big \rangle 
\geq (1-\alpha) \sum_{i = 1}^t \mu_{b_i}^i.
\end{equation}
\end{small}

\noindent While it is not possible to prove that~\eqref{eq:clucb_condition2} is always tighter, in the next section we provide an extensive discussion on the potential improvements.

A second limitation of \clinucb is its two-step approach to action selection, in which either the optimistic action satisfies the conservative condition or the baseline $b_t$ is selected. The idea behind this strategy is that it is necessary to select baseline actions before exploring any other action in order to build a ``conservative budget'', which allows performing effective exploration later on (once the conservative condition is met). In \cclinucb we propose to combine the explorative and conservative requirements by selecting the most optimistic (i.e., useful for exploration) ``safe'' (i.e., satisfying the conservative condition) action. Formally, the algorithm computes the set $\mathcal{C}_{t}$ of ``safe'' arms  such that\footnote{$S_{t-1}$ contains all steps when the baseline is \textit{not} selected, and now it may include arms different from the \ucb arm.}
\begin{align}\label{eq:safety_set_definition}
\mathcal{C}_{t} = \Big\{ &a\in \mathcal{A}_{t}\setminus \{ b_{t}\} \mid  \sum_{i \in S_{t-1}} r_{a_{i}}^i - \psi_{L}(t) + \sum_{i \in S_{t-1}^{b}} \mu_{b_i}^i\nonumber \\
&+ \max\big\{\min_{\theta \in \Theta_t} \langle \theta, \phi_{a}^t \rangle, 0\big\}\geq (1 - \alpha)\sum_{i=1}^t \mu_{b_i}^i\Big\}
\end{align}
where $\psi_L$ is the Martingale bound given in Eq.~\ref{eq:martingale}. 
The algorithm plays the arm that solves the following constrained optimization problem:
\begin{equation}
\label{eq:safety_set_linear}
\begin{aligned}
a_t \in \arg\max\Big\{r_{b_{t}}^{t}, \max_{a \in\mathcal{C}_{t}}  \max_{\theta \in \Theta_t} \langle \theta, \phi_a^t \rangle\Big\}
\end{aligned}
\end{equation}
where, by definition, the max over an empty set is $-\infty$. The maximizer is either the baseline arm $b_t$ or an arm in $\mathcal{C}_t$ that is optimistic \wrt the baseline.
%
In order to illustrate the idea behind~\eqref{eq:safety_set_linear}, consider the configuration illustrated in Fig.~\ref{fig:issue_opt}\emph{(left)} for a MAB setting. If the algorithm has not built enough margin, the UCB arm ($a_1$) would not satisfy the conservative condition as its lower-confidence bound is well below the baseline. As a consequence, \cucb selects the baseline arm $a_4$. A direct improvement can be achieved by selecting arm $a_3$ (i.e., the one with the higher lower bound among the arms passing the conservative condition) as suggested in~\cite{WuSLS16}. Nonetheless, while arm $a_3$ is indeed better than baseline and it allows building conservative budget faster, it may not be effective from an exploration point of view. In~\eqref{eq:safety_set_linear} we suggest arm $a_2$ would be a better choice as it does not give up on reducing the regret (i.e., it has a larger UCB than $a_2$). This may indeed result in a better tradeoff between building conservative budget and performing effective exploration. 
Finally, Fig.~\ref{fig:issue_opt}\emph{(right)} shows that~\eqref{eq:safety_set_linear} may be effective even in the linear setting. The stretched out ellipsoid on one axis gives a precise estimate of some bad arms, while good arms would not be selected by choosing the arm maximizing the lower bound. Even though the arms are more correlated due to the linear structure, the interpretation of this case is as the one for stochastic MABs.

From a computational perspective, the cost of an update for both \clinucb and \cclinucb is $O(A d^3)$.
This complexity comes from the maximization over action and the construction of the confidence intervals.
Compared to \clinucb, \cclinucb has to evaluate the conservative condition for each arm instead of only for the \ucb arm.
However, the cost of this operation is dominated by the arm selection procedure.

Finally, we notice that following the same construction as in~\cite{KazerouniGAR17}, \cclinucb can be easily adapted to the case when Asm.~\ref{asm:basline} does not hold and the baseline performance needs to be estimated online. 

\begin{algorithm}[t]
        \DontPrintSemicolon
\KwIn{$\alpha$, $\delta$, $T$}{}
Set $S_0 = S_0^b = \emptyset$, $k=0$\;
 \For{$t = 1, \ldots,n$}{
         Compute ``safe'' set $C_t$ as in Eq.~\ref{eq:safety_set_definition} or Eq.~\ref{eq:checkpoint_safety_set} \;
         Compute $a_t$ by solving~\eqref{eq:safety_set_linear} \;
         Pull arm $a_t$ and observe $r_{a_t}^t$\;
         \eIf{$a_{t} \neq b_{t}$}{
         	Set $S_t = S_{t-1} \cup \{t\}$, $S^b_{t} = S^b_{t-1}$\;
         	Compute new confidence set $\Theta_{t+1}$}
		{
		Set $S_t = S_{t-1}$, $S^b_{t} = S^b_{t-1}$, $\Theta_{t+1} = \Theta_{t}$\;
		}
 \If{$t \bmod T = 0$}{
         $k = k + 1$\;
 }
 }
 \caption{\cclinucb ($T=1$) and \cclinucbc}
 \label{alg:clucb2}
\end{algorithm}


\subsection{Theoretical Analysis}
Let $\Delta_{a}^t = \mu_{a^\star_t}^t - \mu_{a}^t$ be  the action gap at time $t$.
As in~\cite{KazerouniGAR17}, we rely on the following assumption.
\begin{assumption}\label{asm:baseline_performance}
There exists $0\leq \Delta_{l}\leq \Delta_{h}$ and $0< \mu_{l} \leq \mu_{h}$ such that for every $t$ :
\begin{align*}
\Delta_{l}\leq \Delta_{b_{t}}^{t} \leq \Delta_{h} \quad \text{ and } \quad \mu_{l}  \leq \mu_{b_{t}}^{t} \leq \mu_{h} 
\end{align*}
\end{assumption}
Asm.~\ref{asm:baseline_performance} ensures that the baseline policy has a minimum level of performance, which is reasonable since the baseline policy is the strategy currently used by default. 
Note that in MABs and linear bandits  $\mu_l = \mu_h = \mu_b$ since the performance does not depend on a system context. 
The terms $\Delta_l$ and $\mu_h$ are not critical quantities in the regret bound and it is possible to take $\Delta_{l} = 0$, $\mu_{h} = 1$.
We are now ready to state the following result for \cclinucb.

\begin{theorem}\label{thm:regret}
        For any contextual linear bandit problem, under Asm.~\ref{asm:environment},~\ref{asm:basline}, and~\ref{asm:baseline_performance}, \cclinucb satisfies the conservative condition with probability $1-\delta$ and its 
        regret 
        can be bounded for any $n >0$ with probability at least $1-\delta$ by
\begin{align}\label{eq:regret_cclinucb}
R&_{\cclinucb}(n) \leq O\Bigg( \sigma d \log\left(\frac{nD^{2}}{\lambda d}\right) \sqrt{n} \\
&+ \frac{\Delta_h d^{2}}{(\alpha \mu_{l})^{2}}\big( \sqrt{\lambda}B + \sigma\big)^{2} \log\Bigg( \frac{d^{2}\big( \sqrt{\lambda}B + \sigma\big)^{2}\sqrt{D_{0}}}{\sqrt{\delta}\alpha \mu_{l}}\Bigg)^{2} \Bigg)\nonumber
\end{align}
where $D_{0} := \max\big\{ 2D^{2}/\lambda, 3\big\}$.
\end{theorem}

This regret is of the same order as the bound for \clinucb.\footnote{Notice that there is a typo in Thm.6 in~\cite{KazerouniGAR17}, as the denominator in the log term of $K$ should be $1$.} While this shows that the changes made to \clinucb are ``safe'', we cannot prove a direct improvement to the regret performance, apart from better constants (the regret of \clinucb is at least half of the one of \cclinucb, see appendix). Notice also that in the MAB case, this is not even possible in general, as \cucb is already proved to match the lower bound (in a worst-case sense).
However, a worst-case argument may be misleading in the ranking of the algorithms.
The empirical validation reported in the experimental section will provide a more direct evidence of the improvement of \cclinucb over \clinucb.
In the rest of the section we analyze the parts of the regret that are most directly impacted by the two algorithmic changes in \cclinucb.

\textbf{Discussion on martingale bound.}

{\color{black} An} interpretation for the $\sqrt{d}$-improvement comes from comparing~\eqref{eq:clucb_condition} and~\eqref{eq:clucb_condition2}.
The minimization in~\eqref{eq:clucb_condition} has a closed form solution given by $\langle \wh{\theta}_t, x\rangle - \beta_t \|x\|_{V_t^{-1}}$, with $x := \sum_{i \in S_{t-1}} \phi_{a_i}^i$.\footnote{We remove the contribution of the optimistic arm $a_t$ since it is the same when using martingale or self-normalizing bound.} While  $\langle \wh{\theta}_t, x\rangle \approx \sum_{i \in S_{t-1}} r_{a_i}^i$ since $\wh{\theta}_t$ solves the associated regularized least-square problem, $\beta_t \|x\|_{V_t^{-1}}= \wt{O}(\sigma\sqrt{d|S_{t-1}|})$, which is larger than the martingale term, which is of order $\wt{O}(\sigma \sqrt{|S_{t-1}|})$. The advantage of the martingale argument is that it avoids to explicitly use the linear structure of the reward by building a concentration for the sum of scalar values.
As shown above, this allows to derive a bound independent from the dimensionality of the linear parametrization. Nonetheless, notice that in evaluating the quality of the next arm, a minimization over $\theta$ is still needed in~\eqref{eq:clucb_condition2}, which brings back the dependency on $\sqrt{d}$ (but on a much smaller term) in the regret analysis, which eventually prevents us from proving an explicit advantage in the final bound. A similar reasoning can be derived for the MAB case (see appendix).

{\color{black} Another interpretation for this $\sqrt{d}$-improvement can be seen when looking at the regret.}
{\color{black} 
We start bounding the regret as:
\[
        R_{\cclinucb}(n) \leq \sum_{t \in S_n} \left(\mu_{a^\star}^t - \mu_{a_t}^t\right) + |S^b_n| \Delta_h
\]
In~\cite{KazerouniGAR17}, the first term is upper-bounded by the standard regret of \linucb, while the key step is to bound the regret $|S^b_n| \Delta_h$ incurred while playing the baseline. 
By exploiting the martingale bound, we can provide a tighter bound for $|S^b_n|$ compared to \clinucb.
In~\cite{KazerouniGAR17} (after Eq. 19), the authors shows that $\alpha \mu_l |S^b_n| \lesssim 114 d^2 c_1^2 / (\alpha \mu_l)$ where $c_1 = \sigma + \sqrt{\lambda}B$ (ignoring logarithmic terms).
By exploiting the martingale bound, we can show that $\alpha \mu_l |S^b_n| \lesssim 10 (\sigma + d c_1)/\sqrt{\alpha\mu_{l}} + 32 (\sigma + d c_1)^2 / (\alpha \mu_l)$ (see Eq.~\ref{eq:improvement_over_kaz}).
This already shows that we have a linear term in $d$ depending only on $1/\sqrt{\alpha}$ and a term quadratic in $d$ as in \clinucb with a much smaller constant. This is a big improvement compared to \clinucb that shows that the martingale indeed provides a $\sqrt{d}$-improvement (and also in $1/\sqrt{\alpha}$). Finally, if we take a very loose upper bound we obtain a term that is smaller by a factor at least $2$ compared to the one of CLUCB (formally we obtain that $\alpha \mu_l |S^b_n| \lesssim 48 d^2 c_1^2 / (\alpha\mu_l)$).
}

\textbf{Discussion on action selection.} The second difference in \cclinucb is the action selection process.  
Denote by $\wh{\mu}_i$ the empirical mean of arm $i$, let $\mathcal{A}^{\ucb}_t := \argmax_{i \in [K]} \{ \wh{\mu}_i + \psi^{\ucb}_t(i) \}$ be the set of UCB optimal arms, $\mathcal{A}^+_t := \{i \in [K] : \wh{\mu}_i + \psi^{\ucb}_t(i) \geq \mu^\star\}$ the set of optimistic arms and $\mathcal{C}_t$ the set of arms satisfying the conservative condition.
At any time $t$ we can define three events: $E_{1,t} = \{a_t \in \mathcal{A}^{\ucb}_t \wedge a_t \in \mathcal{C}_t \}$, $E_{2,t} = \{ a_t \neq b \wedge a_t \notin \mathcal{A}^{\ucb}_t \wedge a_t \in \mathcal{C}_t \}$ and $E_{3,t} = \{ a_t = b \}$. Following the two-step selection process of \cucb, only $E_{1,t}$ and $E_{3,t}$ can happen. In case $E_{1,t}$, the algorithm behaves like \ucb, thus performing exploration that contributes to reduce the regret over time. In case $E_{3,t}$, the regret is equal to $\Delta_b$ and no ``progress'' is made on the exploration side, but it serves in building conservative budget for later steps. In \cclinucb, event $E_{2,t}$ happens when the UCB arm is not ``safe'' to play (\ie $\mathcal{A}^{\ucb}_t \notin \mathcal{C}_t$) but there are other arms that are \emph{safe \wrt the baseline}. Interestingly, in this case $a_t$ is indeed \textit{optimistic} w.r.t.\ the baseline, \ie $a_t \in \mathcal{C}_t$ and $\wh{\mu}_i + \psi^\ucb_t(i) \geq \mu_b^{t}$. In analogy with the OFU principle for $a^\star$, this is a good strategy for performing efficient exploration \wrt the baseline policy. 
One source of improvement comes when $a_t$ is optimistic despite not being the UCB arm (\ie $a_t \in \mathcal{A}^+_t$ and $a_t \notin \mathcal{A}^\ucb_t$). In this case, the time step can be analyzed as in \ucb, thus reducing the number of pulls $T_b(n)$ to the baseline arm and its impact to the regret. This is likely to happen in earlier phases of the learning process, where the UCB of all arms tend to be optimistic (i.e., $a_t \in \mathcal{A}^+_t$). Even when $a_t$ is not optimistic w.r.t.\ $\mu^\star$, it may happen that $\mu_{a_t} > \mu_b$. In this case, while this step can be still considered as a ``conservative'', it would contribute for less than $\Delta_b$ regret. Unfortunately, it is difficult to provide theoretical evidence that such events happen often enough to provably reduce the regret. Nonetheless, the fact that the arm is optimistic \wrt the baseline (\ie $\wh{\mu}_i + \psi^\ucb_t(i) \geq \mu_b$) is sufficient to guarantee that the new action selection strategy is never worse than \cucb.

{\color{black}
We can provide an intuition of the impact of the new arm selection through a simple \emph{example}.
Consider the situation of $N>3$ arms, such that $\mu_i > \mu_{i+1}$, for any $i$. Assume we know the variance of the arms and we use Bernstein inequality {\color{black}in building confidence intervals}. Let $\sigma_2 =0$ and $(\mu_i, \sigma_i)_{i>3}$ such that the probability of being safe and better than arm $2$ is negligible. Then the regret can be decomposed by the normal \linucb term (due to event $E_1$), the pull of an arm that is optimistic \wrt the baseline (\ie event $E_2$) and the conservative play. The number of conservative play is thus further reduced due to event $E_2$, leading to a smaller contribution to the regret. Now, in this specific case, the arm played during event $E_2$ is w.h.p.\ always arm $2$. Since it is better than the baseline, we have a further improvement to the regret. To conclude, in this example, \cclinucb will have a regret strictly better than \clinucb.

}

\section{Checkpoint-based Conservative Exploration}

\cclinucb is designed to get a tighter proxy for~\eqref{eq:clucb_constr}, but it is still required to be conservative at any time $t$. This requirement is often too strict in practice, where the conservative condition may be verified only at some ``checkpoints'' over time. We study the case where the checkpoints are equally spaced every $T$ steps. 
We still assume that Asm.~\ref{asm:baseline_performance} holds and we redefine the conservative condition such that for some $\alpha \in [0,1]$ and $T \in \mathbb{N}^{\star}$ a learning agent must satisfy
\begin{equation}
\label{eq:multi_step_condition}
\forall k > 0, \qquad \sum_{t = 1}^{kT} \mu_{a_{t}}^{t} \geq (1 - \alpha)\sum_{t = 1}^{kT} \mu_{b_{t}}^{t},
\end{equation}
which reduces to~\eqref{eq:clucb_constr} for $T=1$. Knowing that the conservative condition is checked every $T$ steps provides the agent with a leeway that can be used to perform more exploration and possibly converge faster towards the optimal policy. 

We first derive a conservative condition that can be evaluated at any time $t \in[ kT + 1, (k+1)T]$ of a phase $k \in \mathbb{N}$ in order to determine whether action $a_t$ is safe.
We build this condition such that when selecting an action $a_{t}$, we want to ensure that by playing the baseline arm until the next checkpoint (\ie until $(k+1)T$) , the algorithm would meet the condition~\eqref{eq:multi_step_condition}.
Formally, at any step $t$, we replace~\eqref{eq:multi_step_condition} with 
\begin{equation}
\label{eq:ideal_algorithm_check}
\begin{aligned}
        \sum_{i\in S_{t-1}} \mu_{a_{i}}^{i} & + \sum_{i\in S_{t-1}^{b}}\mu_{b_{i}}^{i} + \mu_{a_{t}}^{t}  \\
                                            & + \alpha((k+1)T - t)\;\mu_{l} \geq (1-\alpha)\sum_{i = 1}^{t} \mu_{b_{i}}^{i},
\end{aligned}
\end{equation}
where $\mu_l$ is as in Asm.~\ref{asm:baseline_performance} {\color{black} i.e a lower bound on the average reward of the baseline strategy}.

We now modify \cclinucb to satisfy this constraint. Changing the conservative condition impacts how the algorithm evaluates whether an action is ``safe" or not (i.e., if selecting a specific action is compatible with the conservative condition), however the algorithm can still use the bound in~\eqref{eq:martingale} to lower bound the sum of the values of actions selected so far, and compute the conservative set at time $t$ as 
\begin{align}
\label{eq:checkpoint_safety_set}
        \mathcal{C}_{t} := \Bigg\{a &\in \mathcal{A}_{t}\setminus\{b_{t}\} \mid \max\bigg\{\sum_{i\in S_{t-1}} r_{a_{i}}^{i} - \psi_{L}(t),0\bigg\}  \nonumber \\ 
                                    &+\sum_{ i \in S_{t-1}^{C}} \mu_{b_{i}}^{i}+ \max\Big\{\min_{\theta\in\mathcal{C}_{t}} \left\langle \theta, \phi_{a}^{t} \right\rangle, 0 \Big\}  \\ 
                                    &\quad{}+  \alpha\;\big((k+1)T- t\big)\;\mu_{l}\geq (1-\alpha)\sum_{i = 1}^{t} \mu_{b_{i}}^{i}\Bigg\}, \nonumber
\end{align}
and the arm to pull is obtained by solving the constrained problem~\eqref{eq:safety_set_linear}.
We now proceed by analyzing how this different conservative condition impacts the final regret of \cclinucb, which we rename \cclinucbc to stress the checkpoint-based conservative condition.

\begin{theorem}\label{prop:batch_regret}
{\color{black}For any $\delta>0$, with probability as least $1-\delta$} \cclinucbc satisfies condition~\eqref{eq:multi_step_condition} at every checkpoint. Furthermore, let $\tilde{T}_{\delta}^{\alpha} :=  \frac{\alpha \mu_{l}}{(1-\alpha) \mu_{h} + \alpha \mu_{l}}T$, then \cclinucbc suffers a regret
\begin{itemize}
\item If $\tilde{T}_{\delta}^{\alpha} \geq C_{b}(\alpha, \mu_{l}, \delta)$

\begin{small}
\begin{align*}
&R(n) \leq O\bigg( \sigma d \log\left(\frac{nD^{2}}{\lambda d}\right) \sqrt{n} + \\
&\frac{\Delta_{h}}{\alpha \mu_{l}} \max\Bigg\{d\big( \sqrt{\lambda} B + \sigma\big) \sqrt{\tilde{T}_{\delta}^{\alpha} \log\bigg( \frac{\tilde{T}_{\delta}^{\alpha}  D^2}{\lambda \delta}\bigg)}  + \mu_{h}  - \frac{\alpha \mu_{l}}{2}\tilde{T}_{\delta}^{\alpha} \nonumber ,0\Bigg\}, \nonumber
\end{align*}
\end{small}

\item otherwise
\begin{align*}
&R(n) \leq O\bigg(\sigma d \log\left(\frac{nD^{2}}{\lambda d}\right) \sqrt{n}\nonumber\\
&+\frac{\Delta_{h}d^{2}}{(\alpha \mu_{l})^{2}}\big( \sqrt{\lambda}B + \sigma\big)^{2}\log\left( \frac{d^{2}\sqrt{D_{0}}}{\sqrt{\delta}\alpha \mu_{l}}\big( \sqrt{\lambda}B + \sigma\big)^{2}\right)^{2} \bigg),
\end{align*}
\end{itemize}
where
\begin{align*}
C_{b}(\alpha, \mu_{l}, \delta) :=28d^{2}\left(\frac{2/3 + \sqrt{\lambda}B + 2\sigma}{\alpha \mu_{l}}\right)^{2}\times\nonumber&\\
\ln\left( \frac{696d^{2}D_{0}}{\delta\left(\alpha \mu_{b}\right)^{2}}\left(2/3 + \sqrt{\lambda}B + 2\sigma\right)^{2}\right)^{2},& 
\end{align*}
with $D_{0} := \max\{3, 2D^{2}/\lambda\}$.
\end{theorem}

This bound illustrate how the length $T$ of each phase may significantly simplify the problem. As $T$ gets larger, satisfying condition \eqref{eq:multi_step_condition} becomes easier, the baseline is selected less often and more time is spent in exploring different actions, thus leading to smaller regret. Interestingly, when $T$ is large enough (i.e., $T \geq C_{b}(\alpha, \mu_{l}, \delta)$), the conservative contribution has a smaller and smaller impact onto the regret, to the point that the \textit{max} in the second term can become $0$, thus reducing the regret to the standard regret bound of \linucb, with no impact from the conservative constraint.

\textbf{Alternative checkpoint schemes.} While we assumed $T$ to be fixed, it is possible to generalize this result along different lines. If the time between any two checkpoints is known to be lower-bounded by $T_{\min}$, the same analysis could hold by replacing $T$ with $T_{\min}$. Similarly, if $T$ is random from a known distribution, then it is possible to compute the $1 - \delta/2$ quantile of $T$ to recover high-probability guarantees on the conservative properties of the algorithm. Finally, if the checkpoints are completely arbitrary (or even adversarially chosen) then in order to guarantee that~\eqref{eq:multi_step_condition} is verified at \textit{all} checkpoints, the agent needs to be conservative at every step, thus reducing to condition~\eqref{eq:clucb_constr}.

\section{Experiments}\label{sec:experiments}
In this section we provide empirical evidence of the advantage of the Martingale lower-bound and the action selection process in synthetic and real-data problems. 

\subsection{Synthetic Environments}
We consider a MAB with $K=10$ Bernoulli arms whose means are drawn from a uniform distribution, $\mu_i \sim \text{Uniform}([0.25, 0.75])$. 
The conservative level $\alpha$ is set to $0.05$, the horizon $n$ to $10^6$ ($T=1$) and $\delta=0.01$.
We generated $70$ different Bernoulli bandit problems (\ie values of $\mu_i$) and we performed $40$ simulations for each. In each problem, we selected the $4$\textsuperscript{th} best arm as baseline. Out of the $70$ problems, we report the regret curves for the instance where the advantage of \ccucb w.r.t.\ \cucb in terms of the average regret at $n$ is the smallest. This provides an estimated worst-case scenario for our comparison (see Appendix for further details and results). We report the performance of \ucb and an oracle variant of \cucb (\cucb-Or) where the conservative condition~\eqref{eq:clucb_constr} is checked exactly. Furthermore, we test \cucb and a variant of \cucb (\cucb-L) using the action selection process suggested in~\cite{WuSLS16}, which returns the safe arm with the largest lower bound . Finally, we report an ablation study for \ccucb, where we consider the Martingale lower bound (\cucb-M) and the constrained action selection process~\eqref{eq:safety_set_linear} (\cucb-S) separately beside the full algorithm (\ccucb). Fig.~\ref{fig:synthetic}(\emph{top}) shows that the MDS bound alone provides a significant improvement, where the regret is reduced by $43\%$ w.r.t.\ \cucb's. Interestingly, the action selection process (\cucb-S) is much more effective than \cucb-L and it reduces the regret of \cucb by $12\%$. Finally, the combination of the two elements (\ccucb) leads to a reduction of the original regret of more than $51\%$, with a performance which gets much closer to \cucb-Or.


\begin{figure}[t]
    \centering
    \begin{subfigure}[b]{0.49\textwidth}
        \includegraphics[width=0.95\textwidth]{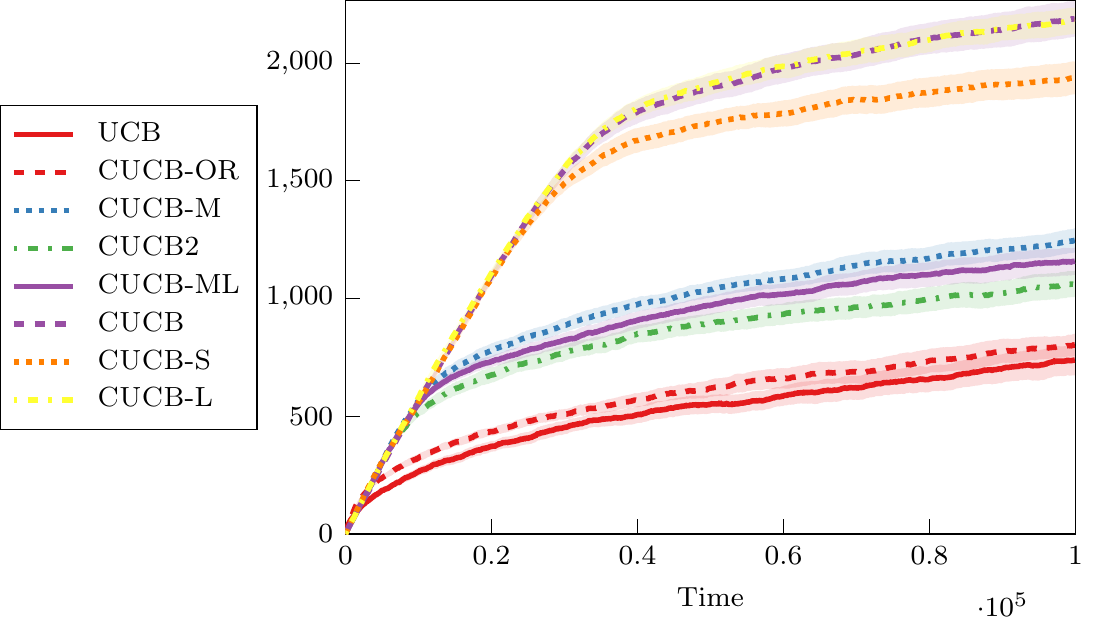}
    \end{subfigure}\\[4pt]
    \begin{subfigure}[b]{0.49\textwidth}
        \includegraphics[width=.95\textwidth]{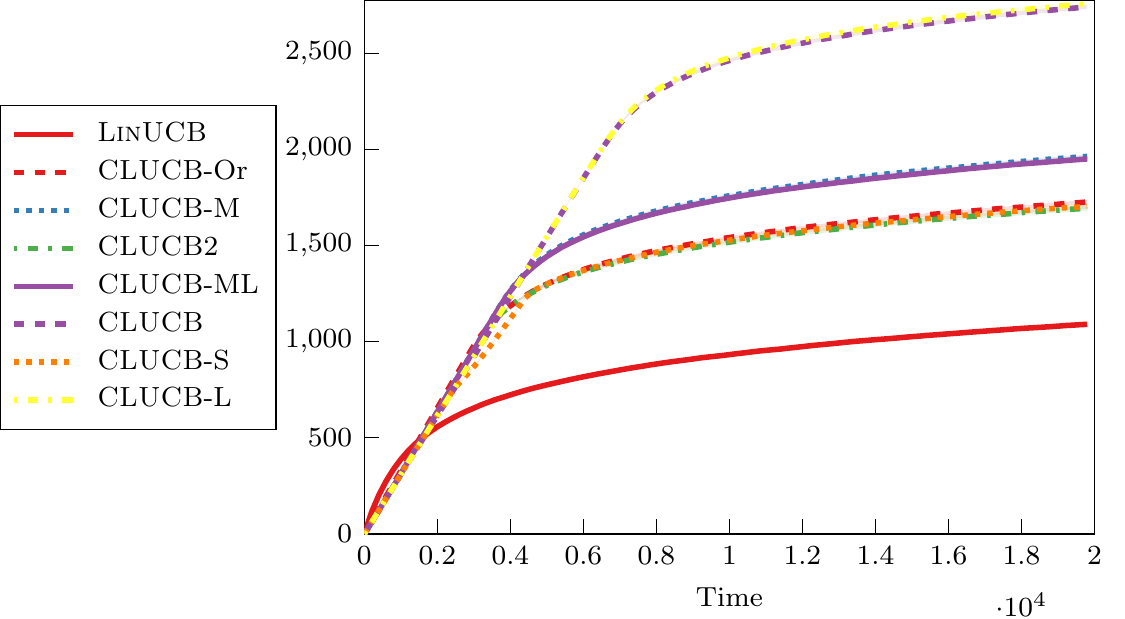} 
    \end{subfigure}
    \caption{Cumulative regret in synthetic models. \emph{Top:} Bernoulli arms. \emph{Bottom:} linear bandits.
    \tikz[baseline]{\node (fnode) {};}}
    \label{fig:synthetic}
\tikz[overlay]{
        \coordinate (c1) at ($(fnode) + (-5.8cm, 6.4cm)$); 
        \node[anchor=north west, font=\tiny] at (c1) {Or: oracle};
        \node[anchor=north west, font=\tiny] at ($(c1)+(0,-6pt)$) {M: MDS bound};
        \node[anchor=north west, font=\tiny] at ($(c1)+(0,-12pt)$) {L: arm with max lower bound};
        \node[anchor=north west, font=\tiny] at ($(c1)+(0,-18pt)$) {S: constrained arm selection};
}
\end{figure}

We also evaluated \cclinucb in the linear setting. 
We considered a non-contextual case with $30$ actions, each defined by a  $100$-dimensional feature vector.
The features and $\theta^\star$ are drawn randomly in the unit ball such that the mean reward of each arm is in $[0,1]$.
The reward noise is drawn from $\mathcal{N}(0, 0.1^2)$ and the baseline arm is the $6$\textsuperscript{th} best action.
We set $\lambda=0.5$, $\delta=0.01$ and $\alpha=0.05$. We generated $70$ models and for each model we averaged the results over $40$ runs.
As in the MAB case, we report the results for the model with the smallest advantage for \cclinucb w.r.t.\ \clinucb (Fig.~\ref{fig:synthetic}(\emph{bottom})). 
Contrary to the MAB setting, the main improvement is obtained by \clinucb-S, whose performance matches \cclinucb and even the oracle variant, corresponding to an improvement of $38\%$ compared to \clinucb. As reported in the appendix in the average case (over models), we observe similar behaviors and performance improvements as in the MAB setting.
{\color{black}
Finally, to provide an idea of how many times event $E_2$ occurs, we performed tests in synthetic linear setting (see experiment section) with baseline being the third-best arm.
On average over multiple models, the percentage of time event $E_2$ happened in the first $5000$ steps ($25\%$ of overall time) is $38.7\%$ ($\pm 9\%$). This shows that potentially we have played something better than baseline and for sure we have gained information (in contrast to playing the baseline).}

\textbf{Checkpoint-based Condition.}
We compare the effect of the checkpoint $T$ on the regret of the algorithms.
We report the results for Bernoulli arms, the linear experiments can be found in Appendix.
In this case, the horizon is set to $n=20000$, all the other parameters are unchanged.
We generated $15$ (integer) checkpoint values logarithmic space between $1$ and $n$.
Fig.~\ref{fig:mab_checkpoint} shows the difference in the regret between \ccucbc and \ucb as a function of $T$.
As expected, the difference decreases as $T$ increases since the condition becomes less strong.
Note that even for $T=n$, \ccucbc and \ucb are different since \ccucbc might discard \ucb optimal arms in order to be safe.
In order to have the same behavior, $T$ should be put sufficiently large in order to overcome the pessimistic estimate used in the condition.
Finally, the improvement provided by the new condition is proportional to the quality of the baseline.
The stronger the baseline, the less is the margin for performing better exploration than playing the baseline itself.

\begin{figure}[t]
\centering
\includegraphics[width=0.8\columnwidth]{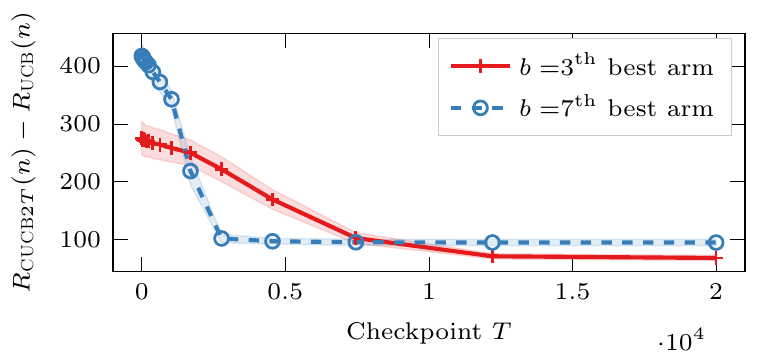}  
\caption{Relative performance between \ccucbc and \ucb in synthetic MAB setting.}
\label{fig:mab_checkpoint}
\end{figure}

\subsection{Dataset-based Environments}
Fig.~\ref{fig:jester} reports the results using the Jester Dataset~\cite{goldberg2001eigentaste} that consists of joke ratings in a continuous scale from $-10$ to $10$ for $100$ jokes from a total of $73421$ users.
We consider the cold start problem: a new user arrives and we need to learn her preferences (\ie $\theta^\star$).
We use the features extracted via a low-rank matrix factorization ($d=35$) to represents the actions (\ie the jokes).
We consider a complete subset of $40$ jokes and $19181$ users rating all the $40$ jokes.
The preference of the new user is randomly selected from the $19181$ users and mean rewards are normalized in $[0,1]$.
The reward noise is $\mathcal{N}(0, 0.1^2)$, the horizon is $T=10^5$, $\alpha=0.01$, $\delta=0.01$ and $\lambda=0.5$ (see appendix).
We report the results averaged over $100$ randomly selected users and for each user we performed $5$ runs.
The baseline is the $10$\textsuperscript{th} best arm.
We also report the regret of \cclinucbc with a checkpoint horizon $T$ equal to $5\%, 10\%$ or $12\%$ of the horizon $n$.
This experiment confirms that \cclinucb performs best, with a regret that is less than half of \clinucb. Furthermore, the results confirm that as the checkpoints become sparser, the performance of \cclinucb approaches the one of \linucb.

\begin{figure}[t]
        \centering
        \includegraphics[width=\columnwidth]{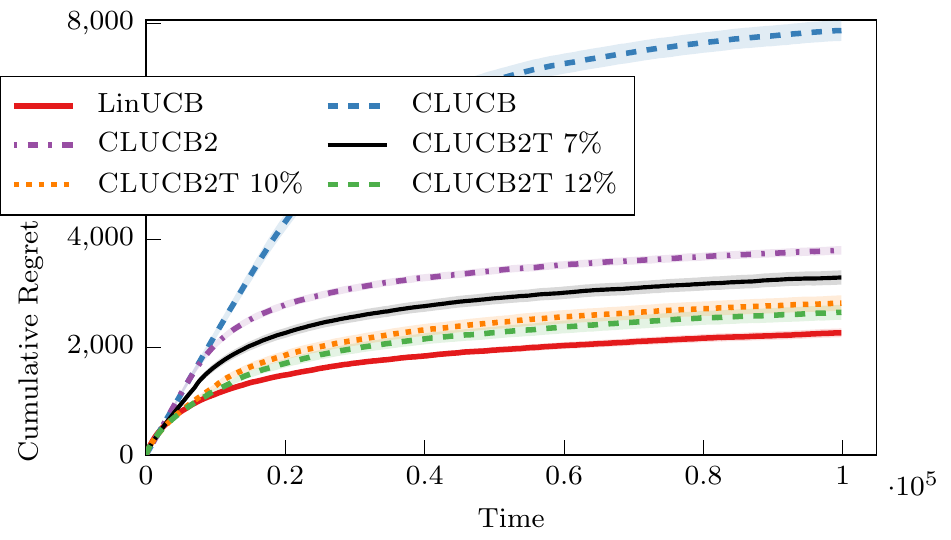}
        \caption{
                Average regret over multiple users of the Jester dataset
        }
        \label{fig:jester}
\end{figure}


\section{Conclusion}

We introduced \cclinucb, a novel conservative exploration algorithm for linear bandit that matches existing regret bound and outperforms state-of-the-art algorithms in a number of empirical tests. In this paper, we also proposed a first direction to relax the conservative condition towards a more realistic scenario. Important directions for future work are: identify alternative conservative exploration constraints that are directly motivated by specific applications, extend the current algorithms beyond linear bandit towards the more challenging reinforcement learning setting.

\bibliography{conservative}

\begin{thebibliography}{}

\bibitem[\protect\citeauthoryear{Abbasi-Yadkori, P{\'a}l, and
  Szepesv{\'a}ri}{2011}]{abbasi2011improved}
Abbasi-Yadkori, Y.; P{\'a}l, D.; and Szepesv{\'a}ri, C.
\newblock 2011.
\newblock Improved algorithms for linear stochastic bandits.
\newblock In {\em Advances in Neural Information Processing Systems},
  2312--2320.

\bibitem[\protect\citeauthoryear{Auer, Cesa-Bianchi, and
  Fischer}{2002a}]{Auer02FA}
Auer, P.; Cesa-Bianchi, N.; and Fischer, P.
\newblock 2002a.
\newblock Finite-time analysis of the multiarmed bandit problem.
\newblock {\em Machine Learning Journal} 47:235--256.

\bibitem[\protect\citeauthoryear{Auer, Cesa-Bianchi, and
  Fischer}{2002b}]{auer2002finite}
Auer, P.; Cesa-Bianchi, N.; and Fischer, P.
\newblock 2002b.
\newblock Finite-time analysis of the multiarmed bandit problem.
\newblock {\em Machine learning} 47(2-3):235--256.

\bibitem[\protect\citeauthoryear{Bottou \bgroup et al\mbox.\egroup
  }{2013}]{Bottou13CR}
Bottou, L.; Peters, J.; Quinonero-Candela, J.; Charles, D.; Chickering, D.;
  Portugaly, E.; Ray, D.; Simard, P.; and Snelson, E.
\newblock 2013.
\newblock Counterfactual reasoning and learning systems: The example of
  computational advertising.
\newblock {\em Journal of Machine Learning Research} 14:3207--3260.

\bibitem[\protect\citeauthoryear{Freedman}{1975}]{freedman1975tail}
Freedman, D.~A.
\newblock 1975.
\newblock On tail probabilities for martingales.
\newblock {\em the Annals of Probability} 3(1):100--118.

\bibitem[\protect\citeauthoryear{Goldberg \bgroup et al\mbox.\egroup
  }{2001}]{goldberg2001eigentaste}
Goldberg, K.; Roeder, T.; Gupta, D.; and Perkins, C.
\newblock 2001.
\newblock Eigentaste: A constant time collaborative filtering algorithm.
\newblock {\em information retrieval} 4(2):133--151.

\bibitem[\protect\citeauthoryear{Katariya \bgroup et al\mbox.\egroup
  }{2019}]{Katariya2019interleaving}
Katariya, S.; Kveton, B.; Wen, Z.; and Potluru, V.~K.
\newblock 2019.
\newblock Conservative exploration using interleaving.
\newblock In {\em {AISTATS}}, volume~89 of {\em Proceedings of Machine Learning
  Research},  954--963.
\newblock {PMLR}.

\bibitem[\protect\citeauthoryear{Kazerouni \bgroup et al\mbox.\egroup
  }{2017}]{KazerouniGAR17}
Kazerouni, A.; Ghavamzadeh, M.; Abbasi, Y.; and Roy, B.~V.
\newblock 2017.
\newblock Conservative contextual linear bandits.
\newblock In {\em {NIPS}},  3910--3919.

\bibitem[\protect\citeauthoryear{Li \bgroup et al\mbox.\egroup
  }{2010}]{li2010contextual}
Li, L.; Chu, W.; Langford, J.; and Schapire, R.~E.
\newblock 2010.
\newblock A contextual-bandit approach to personalized news article
  recommendation.
\newblock In {\em Proceedings of the 19th international conference on World
  wide web},  661--670.
\newblock ACM.

\bibitem[\protect\citeauthoryear{Mansour, Slivkins, and
  Syrgkanis}{2015}]{mansour2015bayesian}
Mansour, Y.; Slivkins, A.; and Syrgkanis, V.
\newblock 2015.
\newblock Bayesian incentive-compatible bandit exploration.
\newblock In {\em Proceedings of the Sixteenth ACM Conference on Economics and
  Computation},  565--582.
\newblock ACM.

\bibitem[\protect\citeauthoryear{Petrik, Ghavamzadeh, and
  Chow}{2016}]{Petrik16SP}
Petrik, M.; Ghavamzadeh, M.; and Chow, Y.
\newblock 2016.
\newblock Safe policy improvement by minimizing robust baseline regret.
\newblock In {\em Advances in Neural Information Processing Systems},
  2298--2306.

\bibitem[\protect\citeauthoryear{Swaminathan and
  Joachims}{2015}]{Swaminathan15CR}
Swaminathan, A., and Joachims, T.
\newblock 2015.
\newblock Counterfactual risk minimization: Learning from logged bandit
  feedback.
\newblock In {\em Proceedings of The 32nd International Conference on Machine
  Learning}.

\bibitem[\protect\citeauthoryear{Thomas, Theocharous, and
  Ghavamzadeh}{2015a}]{Thomas15HCO}
Thomas, P.; Theocharous, G.; and Ghavamzadeh, M.
\newblock 2015a.
\newblock High confidence off-policy evaluation.
\newblock In {\em Proceedings of the Twenty-Ninth Conference on Artificial
  Intelligence}.

\bibitem[\protect\citeauthoryear{Thomas, Theocharous, and
  Ghavamzadeh}{2015b}]{Thomas15HCP}
Thomas, P.; Theocharous, G.; and Ghavamzadeh, M.
\newblock 2015b.
\newblock High confidence policy improvement.
\newblock In {\em Proceedings of the Thirty-Second International Conference on
  Machine Learning},  2380--2388.

\bibitem[\protect\citeauthoryear{Wu \bgroup et al\mbox.\egroup
  }{2016}]{WuSLS16}
Wu, Y.; Shariff, R.; Lattimore, T.; and Szepesv{\'{a}}ri, C.
\newblock 2016.
\newblock Conservative bandits.
\newblock In {\em {ICML}}, volume~48 of {\em {JMLR} Workshop and Conference
  Proceedings},  1254--1262.
\newblock JMLR.org.

\end{thebibliography}
\bibliographystyle{aaai}
\newpage{}
\onecolumn
\section{Appendix A: Proofs}

\subsection{Proof of Thm.~\ref{thm:regret}}

We define three sets of time steps: $S^\star_{t}$ is the steps before $t$ ($t$ included) where the optimistic arm is in the constraint set of~\eqref{eq:safety_set_linear} and it is thus selected, $S_{t}^{+}$ is the steps where the optimistic arm is not in the constraint set and the algorithm did not select the baseline $b_{t}$, finally let $S_{t}^{b} := [t]\setminus \left( S_{t}^\star \cup S_{t}^{+}\right)$ be the remaining time steps at which the baseline was played. We consider the high-probability event in which $\theta^\star \in \Theta_t$ (i.e., the true linear parameter belongs to the confidence ellipsoid). For any action $a$ and any time $t$, we introduce the notation
\begin{align}\label{eq:notation}
\wh\mu_a^t &:= \left\langle \hat{\theta}_{t}, \phi_{a}^{t}\right\rangle \\
\wt\mu_a^t &:= \left\langle \hat{\theta}_{t}, \phi_{a}^{t}\right\rangle + \beta_{t}\big|\big| \phi_{a_{t}}^{t}\big|\big|_{V_{t}^{-1}}.
\end{align}
The regret can be decomposed as
\begin{align}\label{eq:cccucb_regret_decomposition}
R_{\text{\cclinucb}}(n) &= \sum_{t \in S_{n}^\star} \big(\mu_{a_{t}^{\star}}^{t} - \mu_{a_{t}}^{t}\big)  +\sum_{t \in S_{n}^{+}} \big(\mu_{a_{t}^{\star}}^{t} - \mu_{a_{t}}^{t}\big)  +\sum_{t \in S_{n}^{b}} \big(\mu_{a_{t}^{\star}}^{t} - \mu_{b_{t}}^{t}\big)\nonumber \\
&\stackrel{(a)}{\leq} \sum_{t \in S_{n}^\star} \mu_{a_{t}^{\star}}^{t} - \mu_{a_{t}}^{t}  +\sum_{t \in S_{n}^{+}} \mu_{a_{t}^{\star}}^{t} - \mu_{a_{t}}^{t}  + \Delta_{h}\big| S_{n}^{b}\big|\nonumber \\
 &\stackrel{(b)}{\leq} \sum_{t \in S_{n}^\star} \big(\wt\mu_{a_{t}}^{t} - \mu_{a_{t}}^{t}\big)  +\sum_{t \in S_{n}^{+}} \big(\mu_{a_{t}^{\star}}^{t} - \mu_{b_{t}}^{t}\big) + \sum_{t \in S_{n}^{+}} \big(\mu_{b_{t}}^{t} - \mu_{a_{t}}^{t}\big)  + \Delta_{h}\big| S_{n}^{b}\big|\nonumber \\
&\stackrel{(c)}{\leq} \sum_{t \in S_{n}^\star} 2\beta_{t}\big|\big| \phi_{a_{t}}^{t}\big|\big|_{V_{t}^{-1}} + \sum_{t \in S_{n}^{+}} \big(\mu_{b_{t}}^{t} - \mu_{a_{t}}^{t}\big)  + \Delta_{h}\left(\big| S_{n}^{b}\big| + \big| S_{n}^{+}\big|\right) \nonumber \\
&\stackrel{(d)}{\leq}  \sum_{t \in S_{n}^\star \cup S_{n}^{+}} 2\beta_{t}\big|\big| \phi_{a_{t}}^{t}\big|\big|_{V_{t}^{-1}}  + \Delta_{h}\left(\big| S_{n}^{b}\big| + \big| S_{n}^{+}\big|\right),
\end{align}
where $(a)$ is using the upper bound $\Delta_h$ to the per-step regret of the baseline, $(b)$ follows from the high-probability upper-confidence bounds using in \linucb steps, $(c)$ is using the definition of the confidence ellipsoid and the bound on the baseline regret, and $(d)$ follows by the definition of $S_{n}^{+}$. In fact, at any time $t\in S_{n}^{+}$, an arm different from the UCB arm and the baseline is selected following the action selection in~\eqref{eq:safety_set_linear}. Since the baseline arm belongs to the constraint set by definition, then we have
\begin{align}\label{eq:optimistic.arms}
\mu_{b_{t}}^{t} \leq \left\langle \hat{\theta}_{t}, \phi_{a_{t}}^{t}\right\rangle +\beta_{t}\big|\big| \phi_{a_{t}}^{t}\big|\big|_{V_{t}^{-1}} \leq \mu_{a_{t}}^{t} + 2\beta_{t}\big|\big| \phi_{a_{t}}^{t}\big|\big|_{V_{t}^{-1}}.
\end{align}
Recalling the definition of $\beta_t$ and using Theorem~$2$ in~\cite{abbasi2011improved}, we obtain
\begin{align*}
\sum_{t \in S_{n}^\star \cup S_{n}^{+}} 2\beta_{t}\big|\big| \phi_{a_{t}}^{t}\big|\big|_{V_{t}^{-1}} \leq 4\sqrt{nd\log\left( 1 + \frac{nD^{2}}{\lambda d}\right)}\left[\sqrt{\lambda} B + \sigma\sqrt{2\log\left(\frac{1}{\delta}\right) + d\log\left( 1 + \frac{nD^{2}}{\lambda d}\right)}\right]
\end{align*}
Finally, we need to bound the number of times $\big|S_{n}^{+}\big| + \big|S_{n}^{b}\big|$ the optimistic arm was not selected (i.e., it was not in the constraint set). 

\begin{lemma}\label{lem:wrong_times_regret_ccucb}
	For any $\delta>0$, we have with probability $1-\delta$ that :
	\begin{align}\label{eq:improvement_over_kaz}
	|S_{n}^{b}| + |S_{n}^{+}| \leq \frac{2\mu_{h}}{\alpha\mu_{l}} + \frac{4}{3\alpha\mu_{l}}L_{\alpha}
	 +\frac{4}{\sqrt{3}(\alpha\mu_{l})^{3/2}}\Big(4d(\sqrt{\lambda}B + 2\sigma) + \sqrt{2}\sigma\Big)L_{\alpha}
	+ \frac{2}{(\alpha\mu_{l})^{2}}\Big(4d(\sqrt{\lambda}B + 2\sigma) + \sqrt{2}\sigma\Big)^{2}L_{\alpha}&	
	\end{align}
	where $L_{\alpha}$ is a logarithmic term,
	\begin{align*}
	L_{\alpha} = &\log\Bigg[\frac{\sqrt{D_{0}}}{\sqrt{\alpha\mu_{l}\delta}}
	+ \frac{12\sqrt{D_{0}}}{\sqrt{\delta}}\left(\frac{4d(\sqrt{\lambda}B + 2\sigma) + \sqrt{2}\sigma}{\alpha\mu_{l}}\right)^{2}\log\left( \frac{36\sqrt{D_{0}}}{\sqrt{\delta}\alpha\mu_{l}}\Big(4d(\sqrt{\lambda}B + 2\sigma) + \sqrt{2}\sigma\Big)\right)^{2}\Bigg]\times\\			&\times\log\left( \frac{36}{\alpha\mu_{l}}\sqrt{\frac{D_{0}}{\delta}}\Big(4d(\sqrt{\lambda}B + 2\sigma) + \sqrt{2}\sigma\Big)\right)
	\end{align*}
\end{lemma}

Including this result into the regret decomposition provides the final bound
\begin{align*}
R_{\cclinucb}(n) \leq &4\sqrt{nd\log\left( 1 + \frac{nD^{2}}{\lambda d}\right)}\times \left[\sqrt{\lambda} B + \sigma\sqrt{2\log\left(\frac{1}{\delta}\right) + d\log\left( 1 + \frac{nD^{2}}{\lambda d}\right)}\right]
+\frac{2\mu_{h}}{\alpha\mu_{l}} + \frac{4}{3\alpha\mu_{l}}L_{\alpha} \\
&+\frac{4}{\sqrt{3}(\alpha\mu_{l})^{3/2}}\Big(4d(\sqrt{\lambda}B + 2\sigma) + \sqrt{2}\sigma\Big)L_{\alpha} + \frac{2}{(\alpha\mu_{l})^{2}}\Big(4d(\sqrt{\lambda}B + 2\sigma) + \sqrt{2}\sigma\Big)^{2}L_{\alpha}
\end{align*}
where $D_{0} := \max\big\{ 2D^{2}/\lambda, 3\big\}$.

\begin{proof}[Proof of Lemma~\ref{lem:wrong_times_regret_ccucb}]
	Let $\tau : =\max\big\{ t\in [n] \mid t\in S_{n}^{+} \cup S_{n}^{b} \big\}$, i.e., the last time the optimistic arm was not in the constraint set (and either the baseline or another arm was selected). Let $\wt a_\tau$ the optimistic arm at time $\tau$, since it does not satisfy the constraint, we have
	\begin{align}
	\sum_{t \in S_{\tau-1}} r_{a_{t}}^t - \psi_{L}(\tau) + \max\big\{\min_{\theta \in \Theta_\tau} \langle \theta, \phi_{\wt a_\tau}^\tau \rangle, 0\big\} + \sum_{t \in S_{\tau-1}^{b}} \mu_{b_t}^t \leq (1 - \alpha)\sum_{t=1}^\tau \mu_{b_t}^t.
	\end{align}
	Reordering the baseline terms and recalling that $S_{t} = [t] \backslash S_t^b = S_t^\star \cup S_t^+$, we obtain
	\begin{align}
	\alpha\sum_{t=1}^\tau \mu_{b_t}^t \leq \mu_{b_\tau}^\tau + \sum_{t \in S_{\tau-1}} \mu_{b_t}^t - \sum_{t \in S_{\tau-1}} r_{a_{t}}^t + \psi_{L}(\tau) - \max\big\{\min_{\theta \in \Theta_\tau} \langle \theta, \phi_{\wt a_\tau}^\tau \rangle, 0\big\}.
	\end{align}
	Using $\mu_{b_\tau}^\tau \leq \mu_h$ and since the last term is non-negative, we can further simplify the expression as
	\begin{align}\label{eq:int.step}
	\alpha\sum_{t=1}^\tau \mu_{b_t}^t \leq \mu_h + \sum_{t \in S_{\tau-1}} \big(\mu_{b_t}^t - r_{a_{t}}^t\big) + \psi_{L}(\tau).
	\end{align}
	Using the same Friedman inequality as in the construction of the Martingale lower bound and the fact that whenever the algorithm does not select the baseline, the chosen arm is ``optimistic'' w.r.t.\ the baseline (see~\eqref{eq:optimistic.arms}) we have
	\begin{align}
	\sum_{t \in S_{\tau-1}} \big(\mu_{b_t}^t - r_{a_{i}}^i\big) &\leq \sum_{t \in S_{\tau-1}} \big(\mu_{b_t}^t - \mu_{a_{t}}^t\big) + \psi_L(\tau) \leq \sum_{t \in S_{\tau-1}} 2\beta_{t}\big|\big|  \phi_{a_{t}}^{t}\big|\big|_{V_{t}^{-1}} + \psi_L(\tau)\\
	&\leq 4d\left( \sqrt{\lambda} B + \sigma\right)\sqrt{\big| S_{\tau-1}\big|  + 1}  \log\left( \frac{2 D^{2}}{\lambda\delta}\left( \big| S_{\tau-1}\big| + 1\right)\right) + \psi_{L}(\tau),
	\end{align}
	where the last step follows from Lemma $4$ in~\cite{KazerouniGAR17}. As $\tau = 1 + |S_{\tau-1}^\star| + |S_{\tau-1}^{+}| + |S_{\tau-1}^{b}|$, we can lower-bound the LHS of~\eqref{eq:int.step} as
	\begin{align}\label{eq:lhs}
	\alpha\sum_{t=1}^\tau \mu_{b_t}^t \geq \alpha\mu_l (1 + |S_{\tau-1}^\star| + |S_{\tau-1}^{+}| + |S_{\tau-1}^{b}|) \geq \frac{\alpha\mu_l}{2} (1 + |S_{\tau-1}^{+}| + |S_{\tau-1}^{b}|) + \frac{\alpha\mu_l}{2} (1 + |S_{\tau-1}^\star| + |S_{\tau-1}^{+}|).
	\end{align}
	Plugging these results back into~\eqref{eq:int.step} and using the definition of $\psi_L(\tau)$ we obtain
	\begin{align*}
	\frac{\alpha\mu_{l}}{2}\Big(\big|S_{\tau-1}^{b}\big| + \big|S_{\tau-1}^{+}\big| + 1\Big) \leq &- \frac{\alpha\mu_{l}}{2} \Big( \big| S_{\tau-1}^\star\big|
	+  \big| S_{\tau-1}^{+}\big| + 1\Big) + \mu_{h} + \frac{2}{3}\log\Bigg( \Big(1+ \big| S_{\tau-1}^\star\big|+ \big|S_{\tau-1}^{+}\big| \Big) \frac{D_{0}}{\delta}\Bigg)\\
	&+ \left(4d\left( \sqrt{\lambda} B + \sigma\right) +\sqrt{2}\sigma\right)\sqrt{\big| S_{\tau-1}^\star\big| + \big|S_{\tau-1}^{+}\big| + 1}\log\Bigg( \Big(1+ \big| S_{\tau-1}^\star\big|+ \big|S_{\tau-1}^{+}\big| \Big) \frac{D_{0}}{\delta}\Bigg)
	\end{align*}
	where $D_{0} := \max\big\{ 2D^{2}/\lambda, 3\big\}$. To finish, bounding the number of rounds where the algorithm played we use the following lemma (Lemma~\ref{lem:bound_exp}):
	\begin{lemma}\label{lem:bound_exp}
	For any $x\geq 2$ and $a_{1},a_{2},a_{3},a_{4}>0$ such that $a_{2} \geq 2$, the function $f: \mathbb{R}^{+} \rightarrow \mathbb{R}$ with $f(x) = a_{1}\sqrt{x}\log(a_{2}x) + a_4 \log(a_2 x) - a_{3}x$, is bounded as follows:
	\begin{align*}
	\max_{x\geq 2} f(x) \leq\left( a_{1}\sqrt{\frac{2a_{4}}{a_{3}}} + \frac{8a_{1}^{2}}{3a_{3}}\log\left( \frac{18a_{1}\sqrt{a_{2}}}{a_{3}}\right)\right)\log\left( \frac{\sqrt{a_2}}{e}\sqrt{\frac{2a_{4}}{a_{3}} + \frac{64}{9}\left(\frac{a_{1}}{a_{3}}\right)^{2}\log\left( \frac{18a_{1}\sqrt{a_{2}}}{a_{3}}\right)^{2}}\right)&\\
	 + a_{4}\log\left[ a_{2} \left(\frac{2a_{4}}{a_{3}} + \frac{64}{9}\left( \frac{a_{1}}{a_{3}}\right)^{2}\log\left( \frac{18 a_{1}\sqrt{a_{2}}}{a_3}\right)^{2}\right)\right]&	 \end{align*}
	\end{lemma}
	\begin{proof}
	$f$ is a strictly concave function and goes to $-\infty$ as $x\rightarrow +\infty$ thus it admits a unique maximum, noted $x^{\star}$. Moreover, by putting the gradient of $f$ to zero, we have that:
	\begin{align}\label{eq:max_f}
	a_{3}x^{\star} = a_{1}\sqrt{x^{\star}}\log\left( \sqrt{a_{2}x^{\star}}e\right) + a_{4}
	\end{align}
	Thus injecting equation~\ref{eq:max_f} into the definition of $f$, we have that:
	\begin{align*}
	\max_{x\geq 2} f(x) = f(x^{\star}) \leq \frac{a_{1}}{2}\sqrt{x^{\star}}\log\left( a_{2}x^{\star}\right) + a_{4}\log\left( \frac{a_{2}x^{\star}}{e}\right)
	\end{align*} 
	Finally, using eq.~\ref{eq:max_f} and Lemma $8$ of \cite{KazerouniGAR17}, we get:
	\begin{equation*}
	x^{\star} \leq \frac{4a_{1}^{2}}{a_{3}^{2}}\log\left( \frac{4a_{1}\sqrt{a_{2}}e}{a_3}\right)^{2} + \frac{2a_{4}}{a_{3}}
	\end{equation*}
	Hence, putting everything together we have that:
	\begin{align*}
	\max_{x\geq 2} f(x) \leq \left( a_{1}\sqrt{\frac{a_{4}}{2a_{3}}} + \frac{a_{1}^{2}}{a_{3}}\log\left( \frac{18a_{1}\sqrt{a_{2}}}{a_{3}}\right)\right)\log\left( \sqrt{a_{2}}\sqrt{\frac{2a_{4}}{a_{3}} +4\left(\frac{a_{1}}{a_{3}}\right)^{2}\log\left( \frac{18a_{1}\sqrt{a_{2}}}{a_{3}}\right)^{2}}\right)&\\
	 + a_{4}\log\left[ a_{2} \left(\frac{2a_{4}}{a_{3}} +4\left( \frac{a_{1}}{a_{3}}\right)^{2}\log\left( \frac{18 a_{1}\sqrt{a_{2}}}{a_3}\right)^{2}\right)\right]&
	\end{align*}
	\end{proof}
	Using the previous lemma, we get
	\begin{align*}
	\frac{\alpha\mu_{l}}{2}\left( \big|S_{\tau-1}^{b}\big| + \big|S_{\tau-1}^{+}\big|\right) \leq \frac{2}{3}\log\left[ \frac{8D_{0}}{3\delta\alpha\mu_{l}} + \frac{4D_{0}}{\delta(\alpha\mu_{l})^{2}}\left(\Big(4d(\sqrt{\lambda}B + 2\sigma) + \sqrt{2}\sigma\Big)\log\left[ \frac{144d\sqrt{D_{0}}}{\sqrt{\delta}\alpha\mu_{l}}\Big(\sqrt{\lambda}B + 4\sigma\Big)\right]\right)^{2}\right]&\\
	+ \mu_{h} +\frac{2}{\sqrt{3\alpha\mu_{l}}}\Big(4d(\sqrt{\lambda}B + 2\sigma) + \sqrt{2}\sigma\Big)\log\Bigg[ \frac{\sqrt{D_{0}}}{\sqrt{\delta}e}\Bigg(\frac{2}{\sqrt{3\alpha\mu_{l}}} + 4\left(\frac{4d(\sqrt{\lambda}B + 2\sigma) + \sqrt{2}\sigma}{\alpha\mu_{l}}\right)^{2}\times&\\
	\times\log\left( \frac{36\sqrt{D_{0}}}{\sqrt{\delta}\alpha\mu_{l}}\Big(4d(\sqrt{\lambda}B + 2\sigma) + \sqrt{2}\sigma\Big)\right)^{2}\Bigg)\Bigg]&\\
	+ \frac{1}{\alpha\mu_{l}}\Big(4d(\sqrt{\lambda}B + 2\sigma) + \sqrt{2}\sigma\Big)^{2}\log\left( \frac{36}{\alpha\mu_{l}}\sqrt{\frac{D_{0}}{\delta}}\Big(4d(\sqrt{\lambda}B + 2\sigma) + \sqrt{2}\sigma\Big)\right)\log\Bigg[\frac{\sqrt{D_{0}}}{\sqrt{\alpha\mu_{l}\delta}}&\\
	+ \frac{12\sqrt{D_{0}}}{\sqrt{\delta}}\left(\frac{4d(\sqrt{\lambda}B + 2\sigma) + \sqrt{2}\sigma}{\alpha\mu_{l}}\right)^{2}\log\left( \frac{36\sqrt{D_{0}}}{\sqrt{\delta}\alpha\mu_{l}}\Big(4d(\sqrt{\lambda}B + 2\sigma) + \sqrt{2}\sigma\Big)\right)^{2}\Bigg]&
	\end{align*} 
	Since neither baseline nor a non-UCB arm will be pulled anymore after $\tau$, the final statement at $n$ follows.
\end{proof}

\subsection{Proof of Theorem~\ref{prop:batch_regret}}\label{app:proof_prop_1}

The proof follows the same regret decomposition as in~Thm.~\ref{thm:regret}
\begin{align}\label{eq:regret_decomposition_lc}
R(n) \leq  &\sum_{t \in S_{n}^\star \cup S_{n}^{+}} 2\beta_{t}\big|\big| \phi_{a_{t}}^{t}\big|\big|_{V_{t}^{-1}} + \Delta_{h}\left(|S_{n}^{b}| + |S_{n}^{+}| \right).
\end{align}
While the first term is exactly the regret of the \linucb algorithm, we need to derive a bound on the number of times a non-UCB arm is selected similar to Lemma~\ref{lem:wrong_times_regret_ccucb}.

\begin{lemma}\label{lem:number_conservative_lc}
Let $\tilde{T}_{\delta}^{\alpha} :=  \frac{\alpha \mu_{l}}{(1-\alpha) \mu_{h} + \alpha \mu_{l}}T$ and :
\begin{align*}
 C_{b}(\alpha, \mu_{l}, \delta) :=  28d^{2}\left(\frac{2/3 + \sqrt{\lambda}B + 2\sigma}{\alpha \mu_{l}}\right)^{2}\ln\left( \frac{696d^{2}D_{0}}{\delta\left(\alpha \mu_{b}\right)^{2}}\left(2/3 + \sqrt{\lambda}B + 2\sigma\right)^{2}\right)^{2} 
\end{align*}
where  $D_{0} := \max\{3, 2D^{2}/\lambda\}$. Then the  number of conservative plays for the algorithm \ref{alg:clucb2} is such that :
\begin{itemize}
\item if $T_{\delta}^{\alpha} \geq C_{b}(\alpha, \mu_{l}, \delta)$ :
\begin{align*}
\frac{\alpha\mu_{l}}{2}\left( \big|S_{\tau}^{b}\big| + \big| S_{\tau-1}^{+}\big|\right) \leq\max\Bigg\{ - \frac{\alpha\mu_{l}}{2} \Big( T_{\alpha}^{\delta} + 1\Big) + \mu_{h}+ 4d\left( \sqrt{\lambda} B + 2\sigma + \frac{2}{3}\right)\sqrt{ T_{\alpha}^{\delta} + 1}
\times\log\left( \frac{D_{0}}{\delta}\Big(T_{\alpha}^{\delta} + 1\Big) \right)  ,0\Bigg\}
\end{align*}
\item else :
\begin{align*}
\frac{\alpha\mu_{l}}{2}\left( \big|S_{\tau}^{b}\big| + \big|S_{\tau-1}^{+}\big|\right) \leq \frac{57d^{2}}{\alpha \mu_{l}}\Big( 2/3 + \sqrt{\lambda}B + 2\sigma\Big)^{2}\log\left( \frac{44d^{2}\sqrt{D_{0}}}{\sqrt{\delta}\alpha \mu_{l}}\left( 2/3 + \sqrt{\lambda}B + 2\sigma\right)^{2}\right)^{2}
\end{align*}
\end{itemize}
\end{lemma}

\begin{proof}

As previously, let's define $\tau$ as the last time the optimistic arm was not in the constraint set, and let $k$ be such that, $\tau \in \llbracket kT + 1, (k+1)T \rrbracket$, \ie the phase to which  $\tau$ belongs and let $\tilde{a}_{\tau}$ be the optimistic arm at time $\tau$. Because this arm does not satisfy the constraint we have  : 
\begin{align*}
\max\Bigg\{\sum_{t\in S_{\tau-1}} r^{l}_{a_{l}} - \psi_{L}(\tau), 0\Bigg\} + \sum_{l \in S_{\tau-1}^{b}} \mu_{b_{l}}^{l}
+\max\Big\{\min_{\theta\in\mathcal{C}_{\tau}} \left\langle \theta, \phi_{a_{\tau}}^{\tau}\right\rangle, 0 \Big\}+ \alpha((k+1)T - \tau)\mu_{l} \leq (1-\alpha)\sum_{l = 1}^{\tau} \mu_{b_{l}}^{l}& \nonumber.
\end{align*}
where $S_{\tau-1} =S_{\tau-1}^{\star}\cup S_{\tau-1}^{+}$. This can be rewritten as :
\begin{align}\label{eq:condition_not_verified_cl}
\alpha \;\tau\; \mu_{l}\leq \sum_{l \in S_{\tau-1}} \mu_{b_{l}}^{l} - r^{l}_{a_{l}}+\psi_{L}(\tau) + \mu_{h} -\alpha\left( (k+1)T - \tau\right)\mu_{l} -\max\Big\{\min_{\theta\in\mathcal{C}_{\tau}} \left\langle \theta, \phi_{a_{\tau}}^{\tau}\right\rangle, 0 \Big\}
\end{align}
Now, $- \alpha\left( (k+1)T - \tau\right)\Delta_{l}\leq 0$ and  $ -\max\Big\{\min_{\theta\in\mathcal{C}_{\tau}} \left\langle \theta, \phi_{a_{\tau}}^{\tau}\right\rangle, 0 \Big\}  \leq 0$. Thus using the same reasoning as in the regret analysis of \cclinucb, we have :
\begin{align}\label{eq:regret_checkpoint_maximization}
\frac{\alpha\mu_{l}}{2}\Big(\big|S_{\tau}^{b}\big| + \big|S_{\tau-1}^{+}\big|\Big) \leq&- \frac{\alpha\mu_{l}}{2} \Big( \big| S_{\tau-1}^{\star}\big| +  \big| S_{\tau-1}^{+}\big| + 1\Big) + \mu_{h}\nonumber\\
&+ 4d\left( \sqrt{\lambda} B + 2\sigma + \frac{2}{3}\right)\sqrt{\big| S_{\tau-1}^{\star}\big| + \big|S_{\tau-1}^{+}\big| + 1}\log\Bigg( \Big(1+ \big| S_{\tau-1}^{\star}\big|
+ \big|S_{\tau-1}^{+}\big| \Big) \frac{D_{0}}{\delta}\Bigg)
\end{align}
where $D_{0} = \max\{3, 2D^{2}/\lambda\}$. Let's define the function 
\begin{align*}
f :  x \mapsto - \frac{\alpha\mu_{l}}{2} x + \mu_{h}+ 4d\left( \sqrt{\lambda} B + 2\sigma + \frac{2}{3}\right)\sqrt{x}\log\Bigg( \frac{D_{0}x }{\delta}\Bigg)
\end{align*}Equation \eqref{eq:regret_checkpoint_maximization} can be rewritten as  : 
\begin{align*}
\frac{\alpha\mu_{l}}{2}\Big(\big|S_{\tau-1}^{b}\big| + \big|S_{\tau-1}^{+}\big|\Big)  \leq f\Big( \big| S_{\tau-1}^{\star}\big| +  \big| S_{\tau-1}^{+}\big| + 1 \Big) 
\end{align*}
but function $f$ has a maximum and computing it gives  :
\begin{align}\label{eq:result_non_batch}
\frac{\alpha\mu_{l}}{2}\left( \big|S_{\tau}^{b}\big| + \big|S_{\tau-1}^{+}\big|\right) \leq  \frac{57d^{2}}{\alpha \mu_{l}}\Big( 2/3 + \sqrt{\lambda}B 
+ 2\sigma\Big)^{2}\log\left( \frac{44d^{2}\sqrt{D_{0}}}{\sqrt{\delta}\alpha \mu_{l}}\left( 2/3 + \sqrt{\lambda}B + 2\sigma\right)^{2}\right)^{2}
\end{align}
and it is attained at $x^{\star}$ such that :
\begin{align*}
x^{\star}& \leq 28d^{2}\left(\frac{2/3 + \sqrt{\lambda}B + 2\sigma}{\alpha \mu_{l}}\right)^{2}\ln\left( \frac{696d^{2}D_{0}}{\delta\left(\alpha \mu_{b}\right)^{2}}\left(2/3 + \sqrt{\lambda}B + 2\sigma\right)^{2}\right)^{2} \\
&:= C_{b}(\alpha, \mu_{l}, \delta) 
\end{align*}

Function $f$ is increasing before $x^{\star}$ and decreasing afterwards. Now, equation \eqref{eq:result_non_batch} is the result obtained by \cite{KazerouniGAR17}. But, at the beginning of the first phase, \ie when $k = 0$, we have that for $t\leq \frac{\alpha \mu_{l}}{(1-\alpha) \mu_{h} + \alpha \mu_{l}}T$ that :
\begin{align}\label{eq:free_exploration}
\alpha \mu_{l} (T - t) \geq (1- \alpha) \mu_{h} \geq (1-\alpha) \sum_{l = 1}^{t} \mu_{b_{l}}^{l}
\end{align} 
Equation \eqref{eq:free_exploration} implies that at the beginning of the algorithm the conservative condition is satisfied for every possible arm. Thus $\big|S_{\tau-1}^{\star}\big| + \big| S_{\tau-1}^{+}\big| \geq \frac{\alpha \mu_{l}}{(1-\alpha) \mu_{h} + \alpha \mu_{l}}T$.
Therefore, if $ \tilde{T}_{\delta}^{\alpha} := \frac{\alpha \mu_{l}}{(1-\alpha) \mu_{h} + \alpha \mu_{l}}T \geq C_{b}(\alpha, \mu_{l}, \delta)$, we can upper-bound $f\Big(\big|S_{\tau-1}^{\star}\big| + \big| S_{\tau-1}^{+}\big| + 1\Big)$ by the value of $f$ evaluated at  $\frac{\alpha \mu_{l}}{(1-\alpha) \mu_{h} + \alpha \mu_{l}}T$ that is to say we have:
\begin{align}
\frac{\alpha}{2}\left( \big|S_{\tau}^{b}\big| + \big| S_{\tau-1}^{+}\big|\right) \mu_{l} \leq & - \frac{\alpha\mu_{l}}{2} \Big( T_{\alpha}^{\delta} + 1\Big) + \mu_{h}+ 4d\left( \sqrt{\lambda} B + 2\sigma + \frac{2}{3}\right)\sqrt{ T_{\alpha}^{\delta} + 1}\log\left( \frac{D_{0}}{\delta}\Big(T_{\alpha}^{\delta} + 1\Big) \right) 
\end{align}
And, on the other hand  $\alpha \;\big|S_{\tau}^{b}\big|\;\mu_{l} \geq 0$, we have :
\begin{align*}
\frac{\alpha}{2}\left( \big|S_{\tau}^{b}\big| + \big| S_{\tau-1}^{+}\big|\right) \mu_{l}  \leq \max\Bigg\{ - \frac{\alpha\mu_{l}}{2} \Big( T_{\alpha}^{\delta} + 1\Big) + \mu_{h}+ 4d\left( \sqrt{\lambda} B + 2\sigma + \frac{2}{3}\right)\sqrt{ T_{\alpha}^{\delta} + 1}\times \log\left( \frac{D_{0}}{\delta}\Big(T_{\alpha}^{\delta} + 1\Big) \right)  ,0\Bigg\}&
\end{align*}

\end{proof}

Combining the result of the lemma with \linucb regret bound provides the final result
	\begin{itemize}
		\item If $\tilde{T}_{\delta}^{\alpha} \geq C_{b}(\alpha, \mu_{l}, \delta)$ :
		\begin{align*}
		R(n) \leq \frac{\Delta_{h}}{\alpha \mu_{l}} \max\Bigg\{ - \frac{\alpha\mu_{l}}{2} \Big( T_{\alpha}^{\delta} + 1\Big) + \mu_{h}+ 4d\left( \sqrt{\lambda} B + 2\sigma + \frac{2}{3}\right)\sqrt{ T_{\alpha}^{\delta} + 1}
\log\left( \frac{D_{0}}{\delta}\Big(T_{\alpha}^{\delta} + 1\Big) \right)  ,0\Bigg\}&  \nonumber \\
		 +  4\sqrt{nd\log\left( 1 + \frac{nD^{2}}{\lambda d}\right)}\times\Bigg[\sqrt{\lambda} B + \sigma\sqrt{2\log\left(\frac{1}{\delta}\right) + d\log\left( 1 + \frac{nD^{2}}{\lambda d}\right)}\Bigg];& \nonumber\\
		\end{align*}
		\item else :
		\begin{align*}
		R(n) \leq 4\sqrt{nd\log\left( 1 + \frac{nD^{2}}{\lambda d}\right)}\Bigg[\sqrt{\lambda} B +\sigma\sqrt{2\log\left(\frac{1}{\delta}\right) + d\log\left( 1 + \frac{nD^{2}}{\lambda d}\right)}\Bigg] \nonumber\\
		+\frac{57\Delta_{h}d^{2}}{(\alpha \mu_{l})^{2}}\left( 2/3 + \sqrt{\lambda}B + 2\sigma\right)^{2}\times\log\left( \frac{44d^{2}\sqrt{D_{0}}}{\sqrt{\delta}\alpha \mu_{l}}\left( 2/3 + \sqrt{\lambda}B + 2\sigma\right)^{2}\right)^{2}.
		\end{align*}
	\end{itemize}

\section{Appendix B: Experiments}\label{sec:experiements}

\subsection{Worst Case Model for Synthetic Data}
In this section, we present the protocol used to choose which model is used to present the results on synthetic data. To generate Figure ~\ref{fig:synthetic}, we have drawn $n_m$ random bandit models on which each algorithm was ran $n_s$ times.

 In order to show the improvement of \ccucb or \cclinucb over their counterparts \cucb and \clinucb, we have selected the model in which the difference between the regret of \ccucb and of \cucb at $n$ is the smallest. More formally, if $R_{\text{\ccucb}}^{m}(n)$ is the empirical regret of \ccucb after $n$ steps in the bandit model $m$ averaged over $n_s$ runs, $R_{\text{\cucb}}^{m}(n)$ the same for \cucb, and $\mathcal{M}$ the set of models used for the experiments. The model, $m^{\star}$, selected for Figure \ref{fig:synthetic} is such that
\begin{align}\label{eq:model_selection}
1 - \frac{R^{m^{\star}}_{\text{\ccucb}}(n)}{ R^{m^{\star}}_{\text{\cucb}}(n)}  = \min_{m\in \mathcal{M}} 1 - \frac{R^{m}_{\text{\ccucb}}(n)}{ R^{m}_{\text{\cucb}}(n)}  
\end{align}

As algorithm \ccucb achieves consistently better regret than \cucb, the quantities involved in \eqref{eq:model_selection} are positive and thus selecting the minimum effectively gives the worst model in terms of improvement \wrt to \cucb. Empirically, the difference between the regrets was indeed positive for each model. We follow the same protocol for the linear case by comparing the performance of \cclinucb and \clinucb.

\subsection{Multi-Armed Bandits}

\begin{figure}[t]
    \centering
    \begin{subfigure}[b]{0.3\textwidth}
        \includegraphics[width=\textwidth]{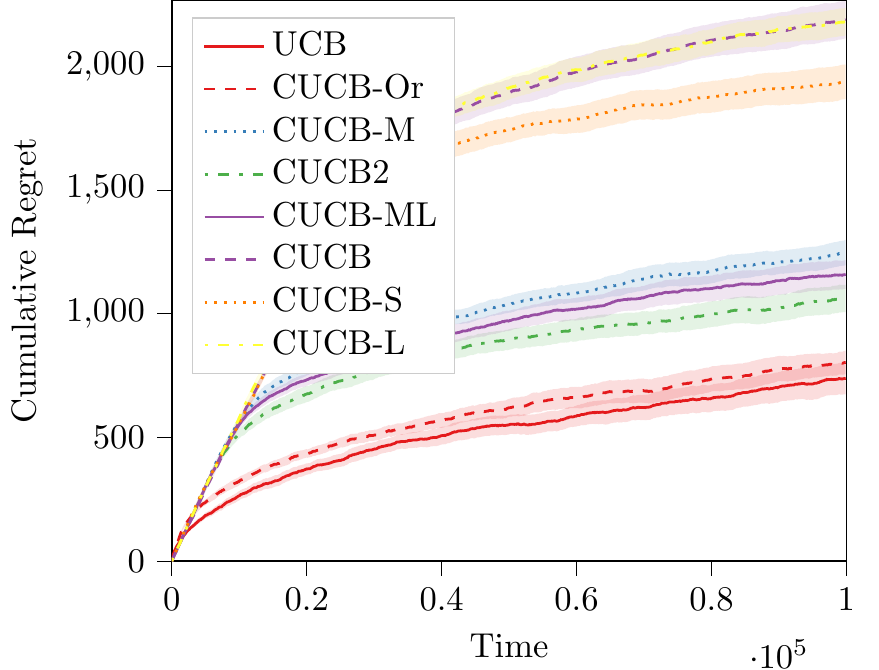}
        \caption{Worst Regret}
        \label{fig:mab_worse_regret}
    \end{subfigure}
    \begin{subfigure}[b]{0.3\textwidth}
        \includegraphics[width=\textwidth]{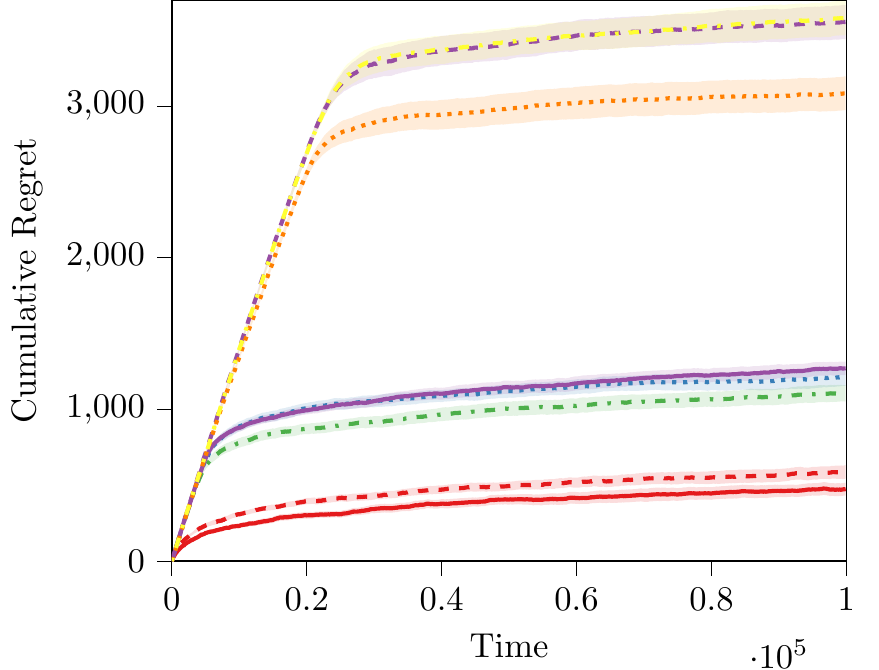}
        \caption{Best Regret}
        \label{fig:mab_best_regret}
    \end{subfigure}
    \begin{subfigure}[b]{0.3\textwidth}
            \includegraphics[width=\textwidth]{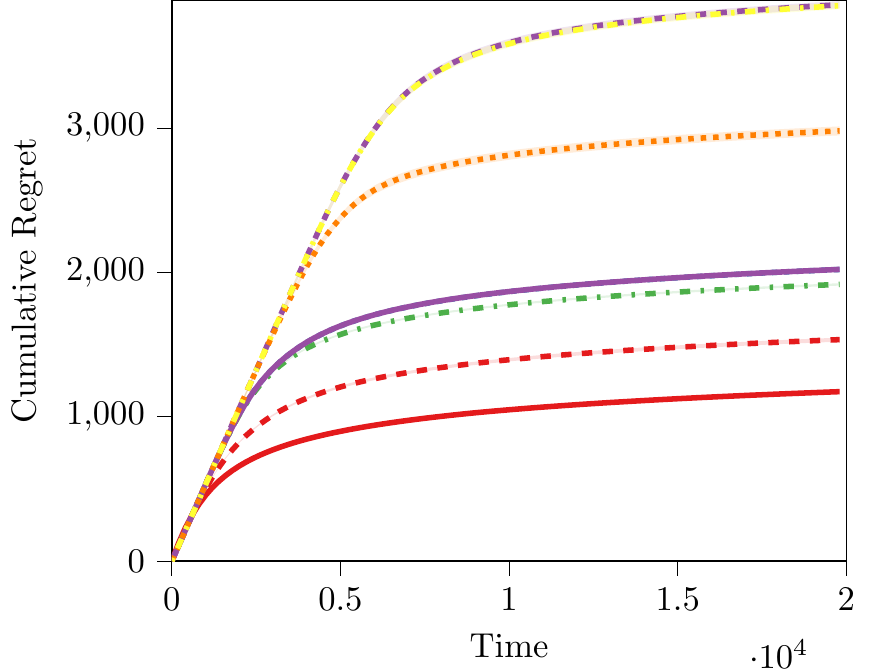}
        \caption{Average Regret}
        \label{fig:mab_avg_regret}
    \end{subfigure}\\
        \begin{subfigure}[b]{0.3\textwidth}
        \includegraphics[width=\textwidth]{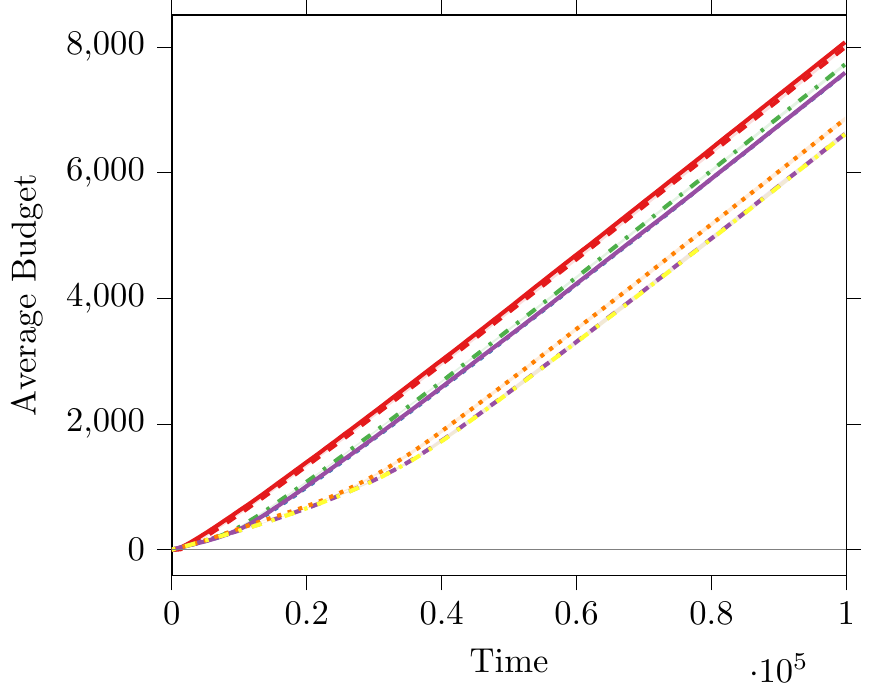}
        \caption{Budget (worst model)}
        \label{fig:mab_worse_margin}
    \end{subfigure}
    \begin{subfigure}[b]{0.3\textwidth}
        \includegraphics[width=\textwidth]{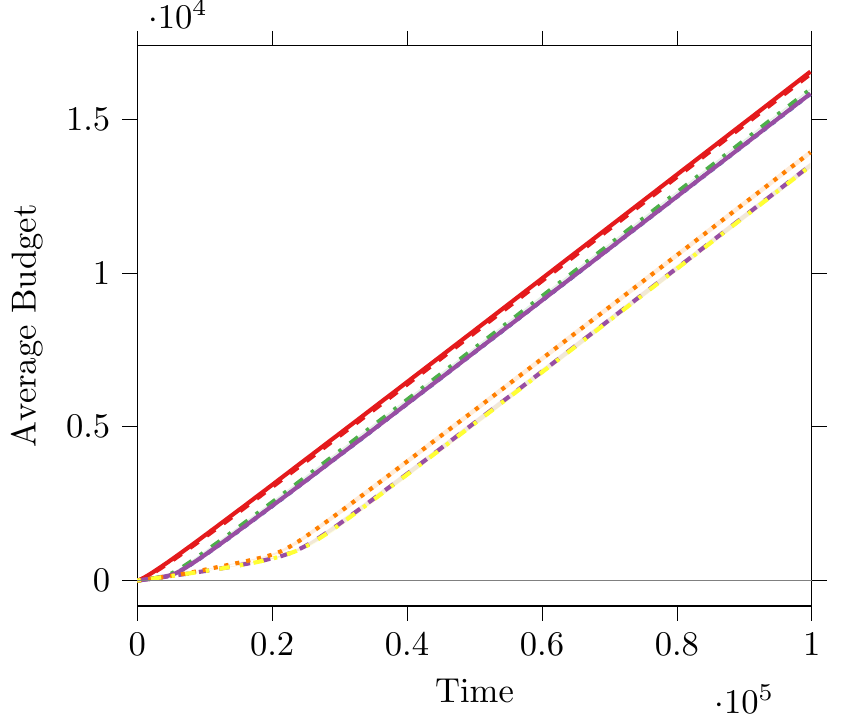}
        \caption{Budget (best model)}
        \label{fig:mab_best_margin}
    \end{subfigure}
    \begin{subfigure}[b]{0.3\textwidth}
            \includegraphics[width=\textwidth]{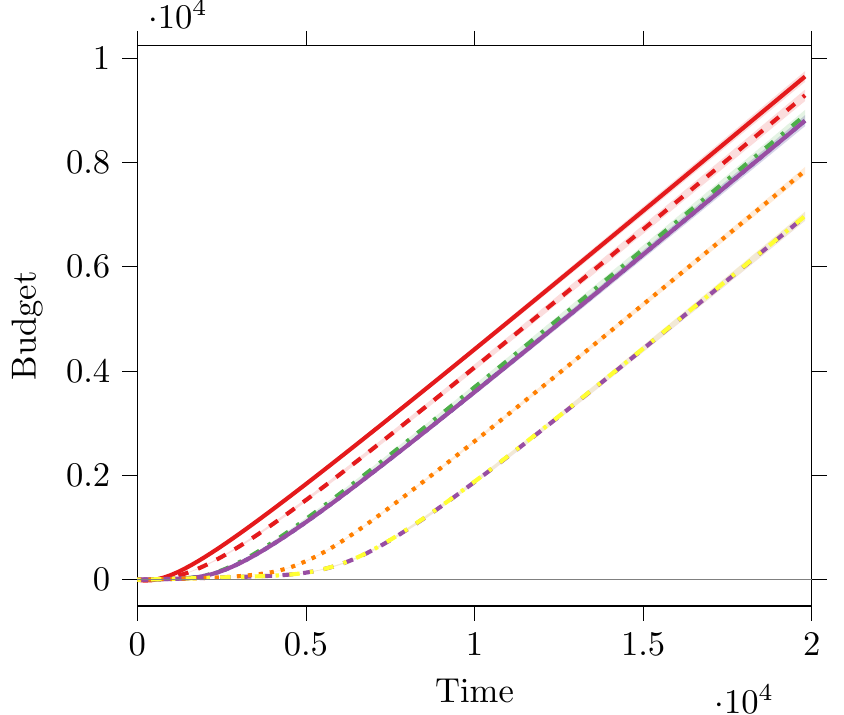}
            \caption{Budget (average model)}
        \label{fig:mab_avg_margin}
    \end{subfigure}\\
    \begin{subfigure}[b]{0.3\textwidth}
        \includegraphics[width=\textwidth]{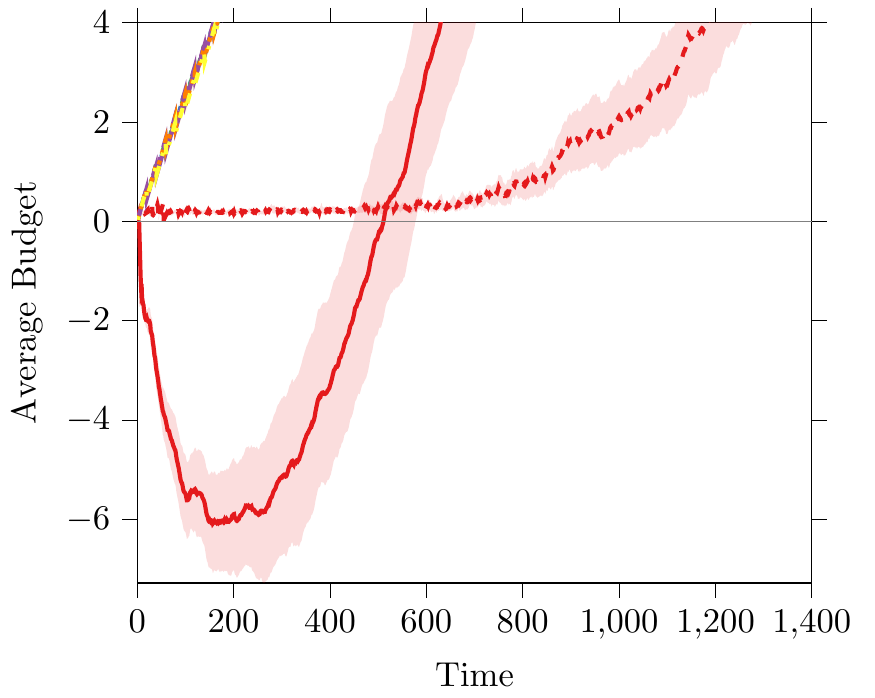}
        \caption{Budget (worst model) zoomed}
        \label{fig:mab_worse_nmargin}
    \end{subfigure}
    \begin{subfigure}[b]{0.3\textwidth}
        \includegraphics[width=\textwidth]{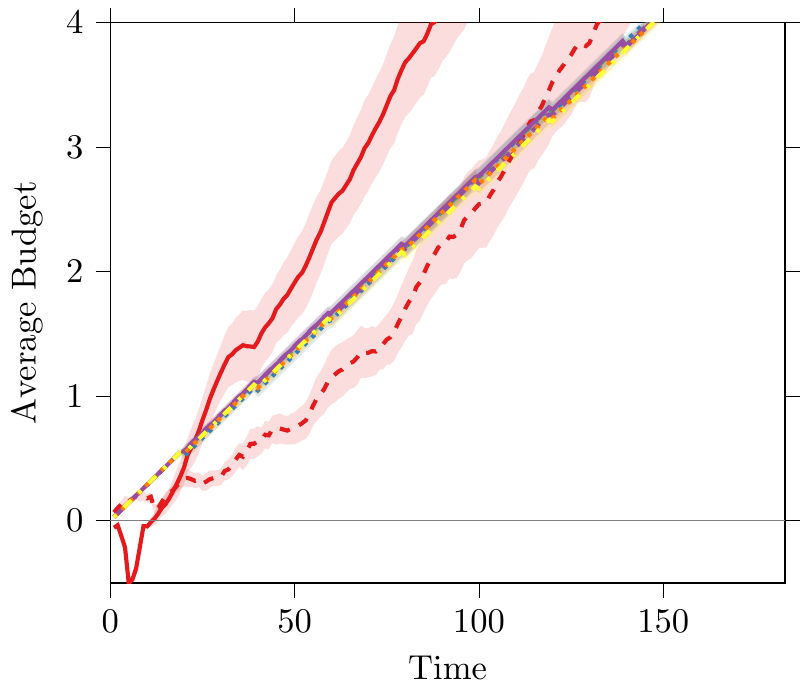}
        \caption{Budget (best model) zoomed}
        \label{fig:mab_best_nmargin}
    \end{subfigure}
    \begin{subfigure}[b]{0.3\textwidth}
            \includegraphics[width=\textwidth]{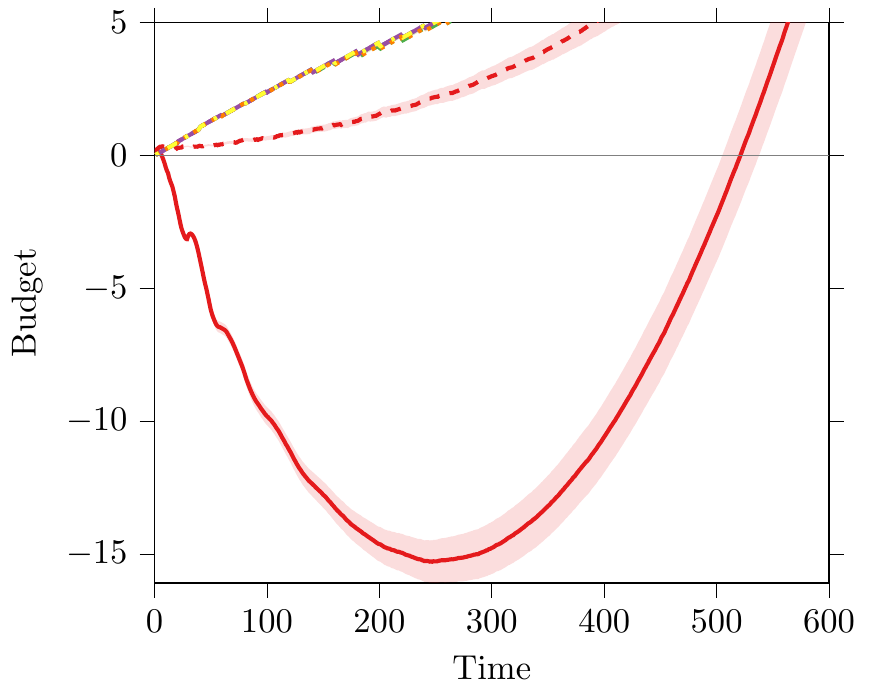}
            \caption{Budget (average model) zoomed}
        \label{fig:mab_avg_nmargin}
    \end{subfigure}
    \caption{We report the regret and the budget for the worst, best and average model in the Bernoulli experiment.}
    \label{fig:bernoulli_regbudget}
\end{figure}

In order to give an idea about the difference of performance across different bandit problems, we report the regret for the worst model, best model and average of model (see Fig.~\ref{fig:bernoulli_regbudget}).
The best model is obtained by changing $\min$ to a $\max$ in Eq.~\ref{eq:model_selection}. 
Notice that the best model and the average one are very similar, meaning that the distribution of the results is very concentrated close to the best one.

We also report the average violation of the conservative constraint. More precisely, for an algorithm, $\mathcal{A}$,  which pulled arms $(a_{1},\hdots, a_{t})$, the \emph{exact} budget is defined as
\begin{align*}
B_{\mathcal{A}}(t) = \sum_{l = 1}^{t} \Big(\mu^{l}_{a_{l}} - (1-\alpha)\mu^{l}_{b}\Big)
\end{align*}
This quantity is what conservative algorithms like \ccucb is constrained to keep positive at every step, while \ccucbc is constrained to keep budget positive at certain predefined checkpoint. On the other hand \ucb does not constraint the budget at all.  
We focus on the time steps where the budget is negative for \ucb and for \ccucbc algorithms.\footnote{Note that the budget is negative since these algorithms are not constrained to be safe uniformly in time but only at the checkpoints.} As it can be noticed the other conservative algorithms trade-off some level of performance (regret) in order to be safe \wrt the baseline. 

Figures~\ref{fig:mab_worse_margin}-\subref{fig:mab_best_margin}-\subref{fig:mab_avg_margin} show the budget of the worst, best and the average over models until $t = n$. Note that all the lines are parallel meaning that all the algorithms have reached nearly optimal policies.
As the reader may notice, the better the algorithm the higher the budget. This is due to the fact that better algorithms tends to explore better (and/or faster) and quickly discard suboptimal arms.
Figure~\ref{fig:bernoulli_regbudget} also reports the time steps where the budget is negative for \ucb. In particular, as pointed out in the introduction, \ucb explores heavily at the beginning which leads to potentially large violation of the constraint, see Fig.~\ref{fig:mab_worse_margin}. On the other hand, the conservative algorithms never violates the conservative constraint.  

\subsection{Synthetic Linear Data}

\begin{figure}[t]
    \centering
    \begin{subfigure}[b]{0.3\textwidth}
        \includegraphics[width=\textwidth]{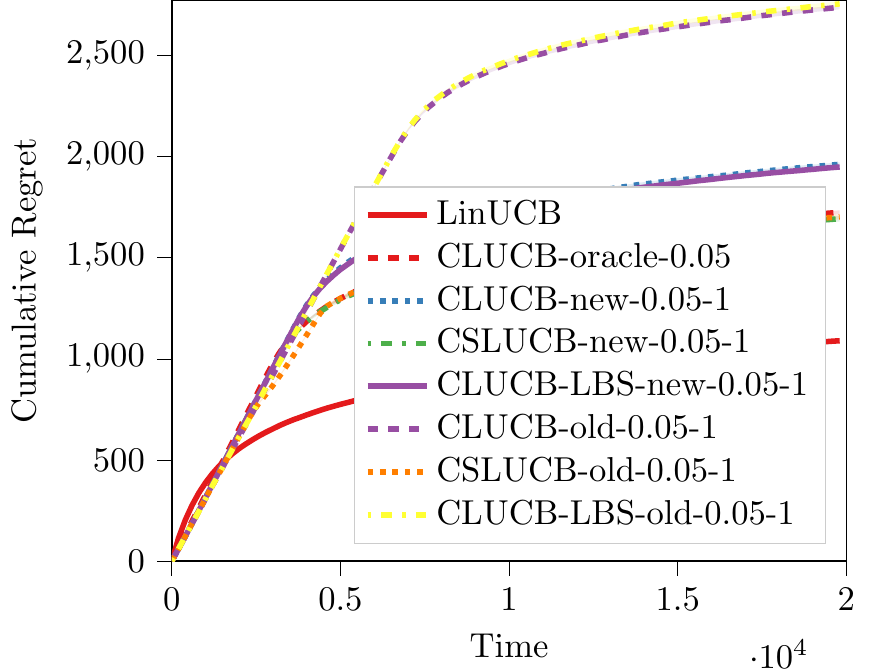}
        \caption{Worst Regret}
        \label{fig:lin_worse_regret}
    \end{subfigure}
    \begin{subfigure}[b]{0.3\textwidth}
        \includegraphics[width=\textwidth]{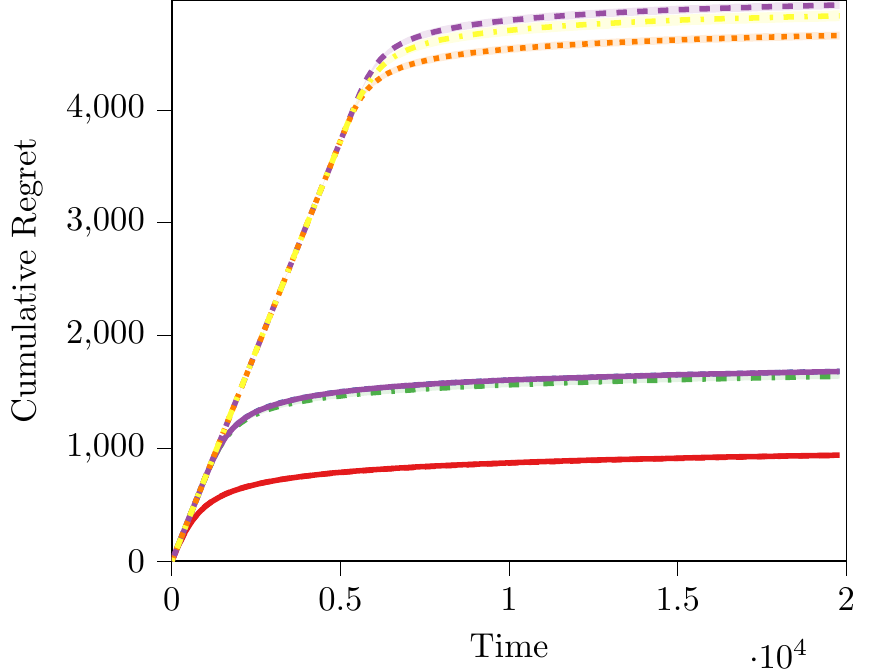}
        \caption{Best Regret}
        \label{fig:lin_best_regret}
    \end{subfigure}
    \begin{subfigure}[b]{0.3\textwidth}
            \includegraphics[width=\textwidth]{./avg_regret.pdf}
        \caption{Average Regret}
        \label{fig:lin_avg_regret}
    \end{subfigure}\\
    \begin{subfigure}[b]{0.3\textwidth}
            \includegraphics[width=\textwidth]{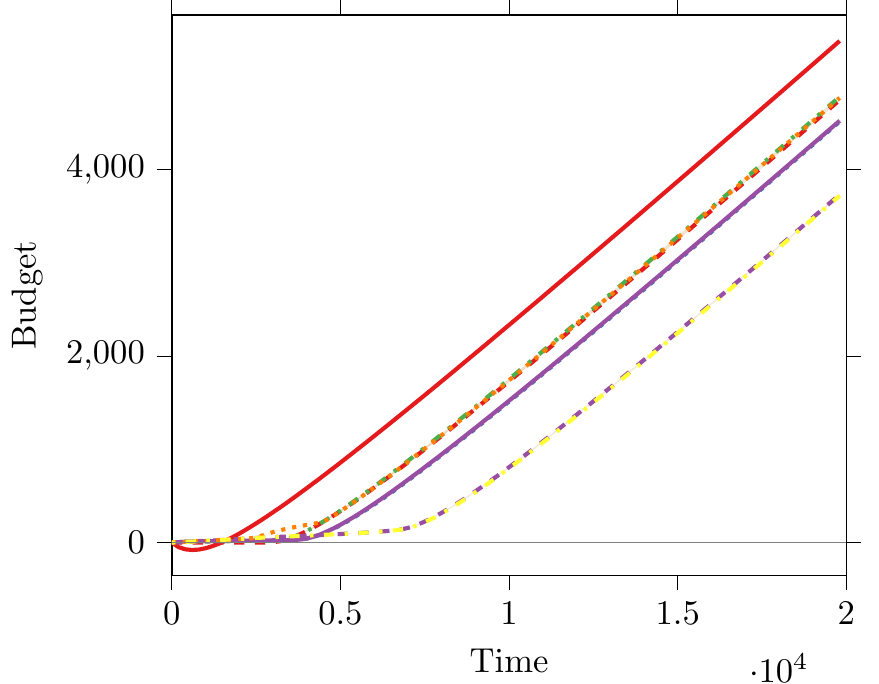}
        \caption{Budget (worst model)}
        \label{fig:linear_worse_margin}%
    \end{subfigure}
    \begin{subfigure}[b]{0.3\textwidth}
        \includegraphics[width=\textwidth]{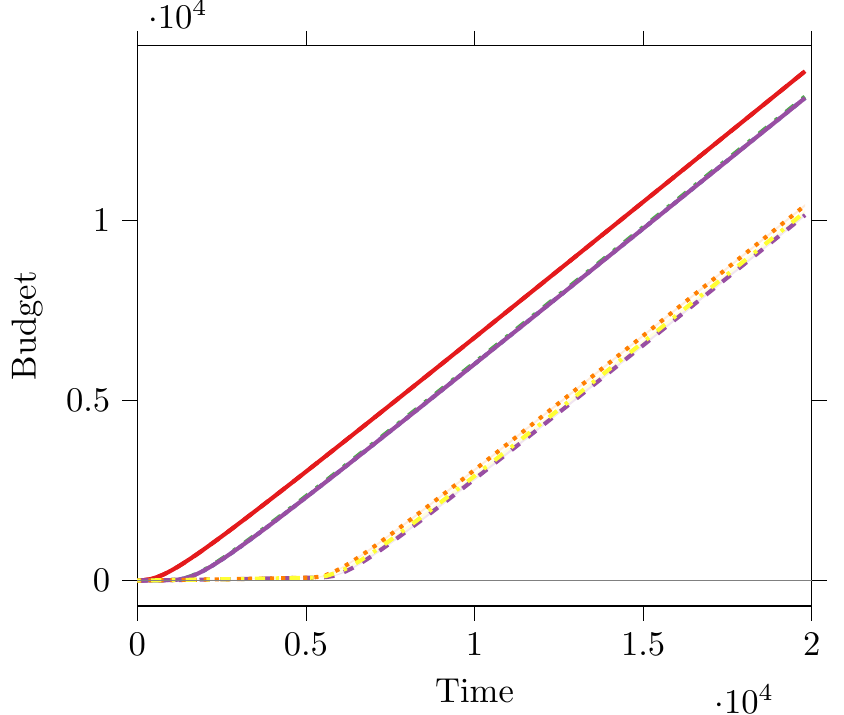}
        \caption{Budget (best model)}
        \label{fig:linear_best_margin}
    \end{subfigure}
    \begin{subfigure}[b]{0.3\textwidth}
            \includegraphics[width=\textwidth]{./avg_margin_long.pdf}
            \caption{Budget (average model)}
            \label{fig:linear_avg_margin}
    \end{subfigure} \\
    \begin{subfigure}[b]{0.3\textwidth}
        \includegraphics[width=\textwidth]{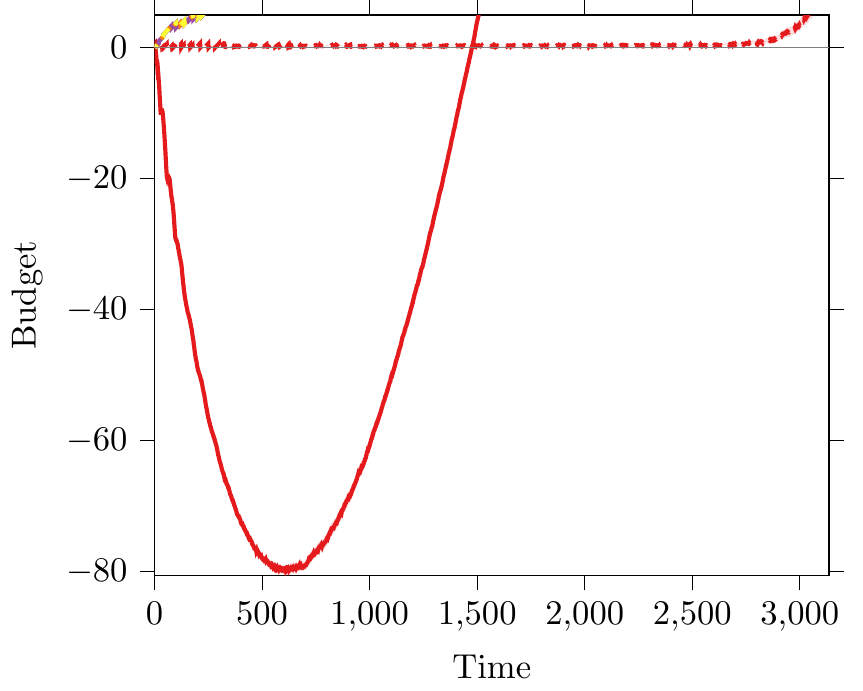}
        \caption{Budget (worst model) zoomed}
        \label{fig:linear_worse_nmargin}
    \end{subfigure}
    \begin{subfigure}[b]{0.3\textwidth}
        \includegraphics[width=\textwidth]{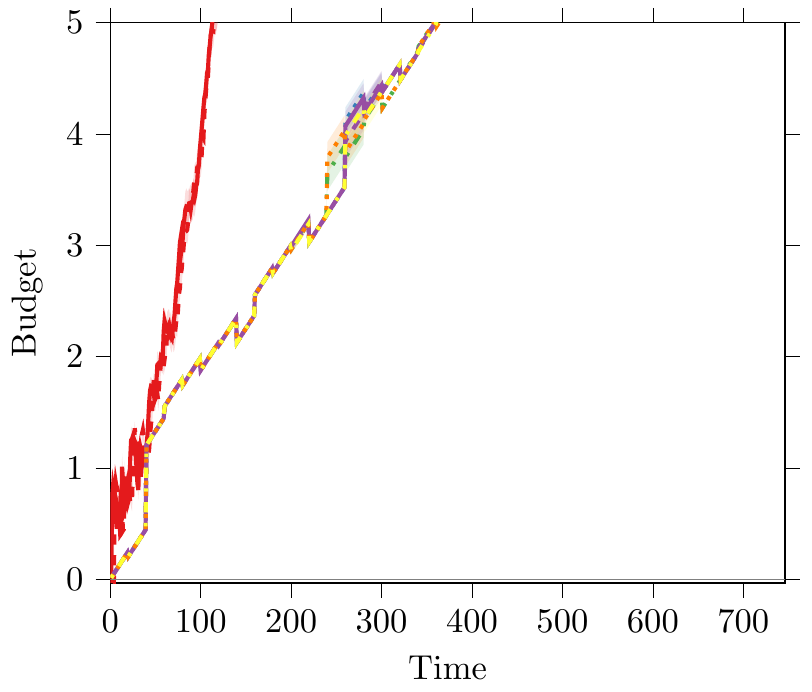}
         \caption{Budget (best model) zoomed}
        \label{fig:linear_best_nmargin}
    \end{subfigure}
    \begin{subfigure}[b]{0.3\textwidth}
            \includegraphics[width=\textwidth]{./avg_margin.pdf}
            \caption{Budget (average model) zoomed}
        \label{fig:linear_avg_nmargin}
    \end{subfigure}
    \caption{We report the regret and the budget for the worst, best and average model in the linear bandit experiment.}
    \label{fig:linear_regbudget}
\end{figure}
 
 We present the regret and budget for the best, average and worst model on different linear bandit problems.  Figure~\ref{fig:linear_regbudget} shows the regret and the budget for \clinucb, \cclinucb and different ablated algorithms used in the experiments. Figure \ref{fig:lin_worse_regret} shows that in the worst case, the algorithms using the action selection process introduced in this paper outperforms \clinucb-OR which can be surprising. However, the latter is an oracle with respect to algorithms using a two stage action selection process. Thus, Figure~\ref{fig:lin_worse_regret} shows that the introduction of a new action selection process can lead to major improvement in regards of the regret of conservative algorithms. The other possible comment is that the introduction of a martingale based concentration inequality does not lead to significant improvement in that specific case. This indicates that in the linear setting, the impact of the choice of the lower bound is less flagrant than for multi-armed bandit because the lower bound used by \clinucb in in some way takes into account of the correlation between arms. 
  
  However, looking at Figure~\ref{fig:lin_avg_regret}, it is clear that in average changing the lower bound does impact positively the regret and that the worst case presented here is a corner case. Figures~\ref{fig:linear_worse_nmargin}-\subref{fig:linear_best_nmargin}-\subref{fig:linear_avg_nmargin} shows the violation of the constraint by \linucb and are similar to the multi-armed bandit case, in the sense that \linucb explore heavily at the beginning of the problem which leads to large violation of the conservative constraint.
  
\subsection{Jester Dataset}
For the Jester~\cite{goldberg2001eigentaste} experiment, we consider the standard linear setting. We performed a matrix factorization of the ratings (after filtering over users and jokes). This provides features for the arms and users, the reward (\ie rating) is the dot product between the arm and user features (we make it stochastic by adding Gaussian noise). We consider a cold start problem where the user is randomly selected at the beginning of the repetition and the agent has to learn the best arm to recommend. When an arm is selected by the algorithm, its reward is computed as the dot product between the arm and user features.

We report the budget $B_{\mathcal{A}}(t)$ in the case of Jester dataset.
We report the average over all the users and simulations.
Fig.~\ref{fig:jester_margin}(\emph{left}) shows that \linucb and the checkpoint-based algorithms violates the \emph{one-step} budget in the initial phase.
\cclinucbc follows the exploratory behaviour of \linucb until the conservative condition~\eqref{eq:multi_step_condition} forces them to revert to a conservative behavior by playing the baseline.
If we observe the long-term behavior (Fig.~\ref{fig:jester_margin}(\emph{right})), all the lines are parallel, meaning that the algorithms have converged to nearly optimal policies. Second, \linucb is the one building the higher budget since it is the first one to converge toward optimal arms. The other algorithms are ordered accordingly to their regret performance.
\begin{figure}[h]
        \centering
        \includegraphics[width=.45\textwidth]{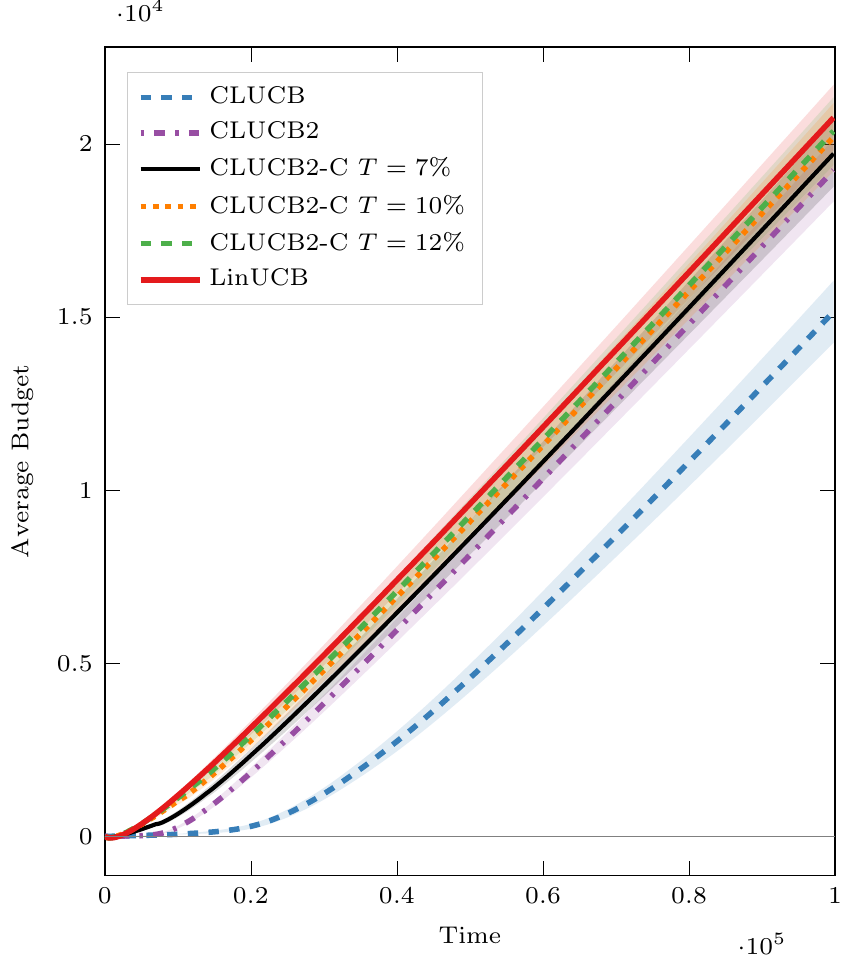}
        \includegraphics[width=.45\textwidth]{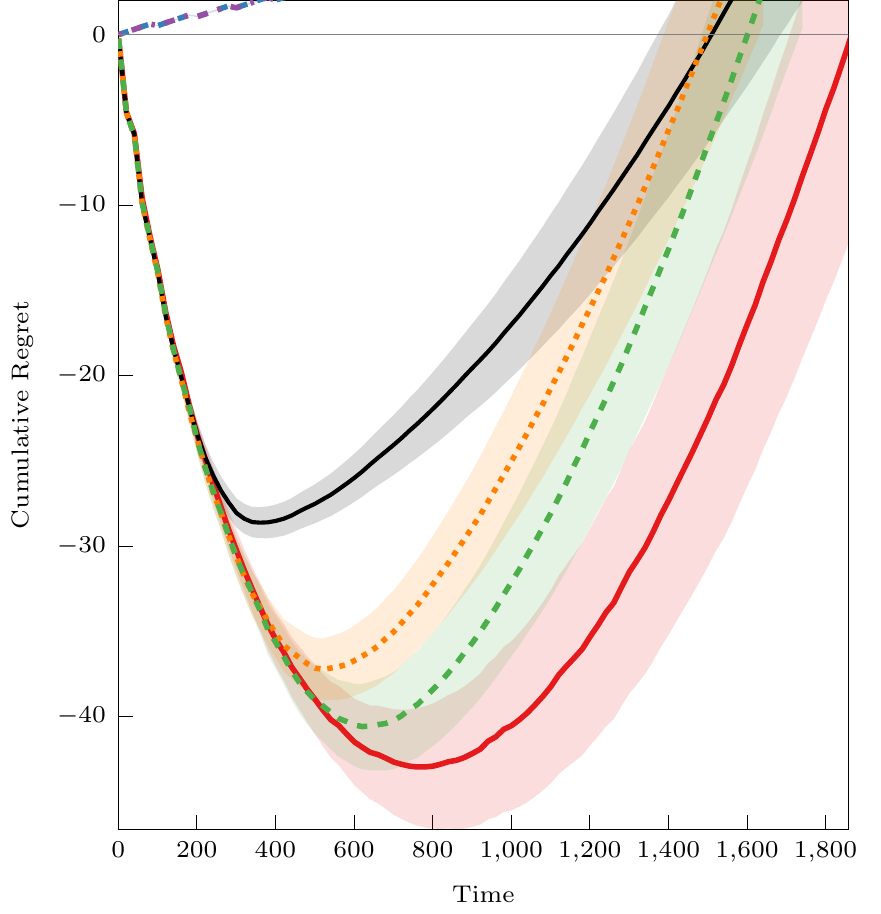}
        \caption{Budget as a function of time for the Jester dataset. The figure shows the average budget over users and repetitions.}
        \label{fig:jester_margin}
\end{figure}

\end{document}